%% file: main.tex
\documentclass[10pt,twocolumn,letterpaper]{article}

\usepackage{cvpr}              


\definecolor{cvprblue}{rgb}{0.21,0.49,0.74}
\usepackage[pagebackref,breaklinks,colorlinks,citecolor=cvprblue]{hyperref}
\usepackage{amsmath,amssymb,amsthm, amsfonts}
\usepackage{algorithm}
\usepackage{algpseudocode}
\usepackage{graphicx}
\usepackage{textcomp}
\usepackage{multirow}
\usepackage{epsfig}
\usepackage{bm}
\usepackage{colortbl}
\usepackage{booktabs}
\usepackage{subcaption}
\usepackage{amsmath}
\usepackage{amssymb}
\usepackage{pifont}
\usepackage{url}
\usepackage[margin=0.7in]{geometry}
\usepackage{framed,multirow}
\usepackage{booktabs}
\usepackage{tabularray}
\UseTblrLibrary{booktabs}
\usepackage{boldline} 
\usepackage{amssymb,amsmath}
\usepackage{romannum}
\AtBeginDocument{\pagenumbering{arabic}}
\usepackage{pifont}

\usepackage{latexsym}
\usepackage{mathrsfs}
\usepackage{float}
\usepackage{soul}
\usepackage{url}
\usepackage{makecell}
\usepackage{subcaption}
\usepackage{adjustbox}
\usepackage{blindtext}
\usepackage{graphicx}
\usepackage{marvosym}
\usepackage{rotating}
\usepackage[listings,skins,breakable]{tcolorbox}
\usepackage[colorlinks]{hyperref}
\usepackage[printonlyused,withpage]{acronym}
\usepackage{caption}
\usepackage{subcaption}
\usepackage{xcolor}       

\newtheorem{theorem}{Theorem}
\definecolor{aliceblue}{RGB}{240,248,255}

\newtheorem{Theorem}{Theorem}
\definecolor{tabhighlight}{HTML}{e5e5e5}
\newcommand{\omark}{\ding{51}} 

\newcommand*{\affaddr}[1]{#1} 

\newcommand{\tablestyle}[2]{\setlength{\tabcolsep}{#1}\renewcommand{\arraystretch}{#2}\centering\footnotesize}
\def\BibTeX{{\rm B\kern-.05em{\sc i\kern-.025em b}\kern-.08em
    T\kern-.1667em\lower.7ex\hbox{E}\kern-.125emX}}
\def\paperID{13843} 
\def\confName{CVPR}
\def\confYear{2026}
\newcommand{\ours}{MedQwen}
\title{Sparse Spectral LoRA: Routed Experts for Medical VLMs}
\author{%
    \makebox[\linewidth][c]{%
    Omid Nejati Manzari\textsuperscript{\Letter}
    \hspace{1em} Hojat Asgariandehkordi
    \hspace{1em} Taha Koleilat
    \hspace{1em} Yiming Xiao
    \hspace{1em} Hassan Rivaz
    }
    \\[1ex]
    \affaddr{Concordia University, Montreal, Canada}
    \\[1ex]  
    \url{https://omid-nejati.github.io/MedQwen/}
    } 

\begin{document}
\maketitle
\input{sec/0_abstract}

\renewcommand{\thefootnote}{}
\footnote{\Letter\ Corresponding Author: \href{mailto:omid.nejatimanzari@mail.concordia.ca}{omid.nejatimanzari@mail.concordia.ca}}
\input{sec/1_intro}
\input{sec/2_method}
\input{sec/3_results}
\input{sec/4_conclusion}
{
    \small
    \bibliographystyle{ieeenat_fullname}
    \bibliography{main}
}

\input{sec/X_suppl}

\end{document}

%% file: sec/0_abstract.tex
\begin{abstract}
Large vision–language models (VLMs) excel on general benchmarks but often lack robustness in medical imaging, where heterogeneous supervision induces cross-dataset interference and sensitivity to data regime (i.e., how the supervisory signals are mixed).  In realistic clinical workflows, data and tasks arrive sequentially, so naive continual training further leads to catastrophic forgetting.
To address these challenges, we propose \ours, a parameter-efficient medical VLM that couples a spectrally routed Mixture-of-Experts (MoE) with a theoretically grounded scaling rule that aligns low-rank updates with a full-rank, fully fine-tuned MoE, without changing the base architecture. Concretely, we initialize each expert from non-overlapping singular value decomposition (SVD) segments of the pretrained weight and introduce a residual compensation and scaling scheme to enable stable expert specialization and consistent routing under distribution shift.
Across 23 medical datasets covering visual question answering, report generation, radiology classification, and hallucination mitigation, \ours\ achieves strong, reliable performance: it approaches full fine-tuning on zero-shot classification with 339$\times$ fewer trainable parameters, and reduces sequential forgetting to $\sim$5\% where strong baselines degrade by $>$20–50\%.

\end{abstract}

%% file: sec/1_intro.tex
\section{Introduction}

Vision-Language Models (VLMs) have demonstrated strong generalization across open-vocabulary tasks by jointly learning from paired images and text~\cite{lu2019vilbert, radford2021learning,li2023llava, moor2023med}.
Despite impressive general-domain results, off-the-shelf VLMs require adaptation to medical tasks \cite{tiu2022expert, bhayana2023performance,hayden2024performance, nisar2024d}. Specialized medical VLMs, including Med-Flamingo~\cite{moor2023med}, HealthGPT~\cite{lin2025healthgpt}, and LLaVA-Med~\cite{li2023llava}, have shown strong performance in medical visual question answering (VQA) and related tasks. However, these models often lack the versatility needed for both discriminative and generative clinical tasks. Moreover, their reliance on large backbones (e.g., the 7B-parameter LLaMA in LLaVA-Med) incurs heavy training and inference costs. Parameter-Efficient Fine-Tuning (PEFT) methods such as Low-Rank Adaptation (LoRA)~\cite{hu2022lora} mitigate these challenges by updating only lightweight low-rank matrices while keeping pre-trained weights frozen, enabling parameter-efficient domain adaptation of medical VLMs~\cite{liu2024improved, yin2023lamm}, but this alone has not resolved the reliability and cost constraints that are important in practice.

\begin{figure}[t] 
    \centering
    \includegraphics[width=.5\textwidth]{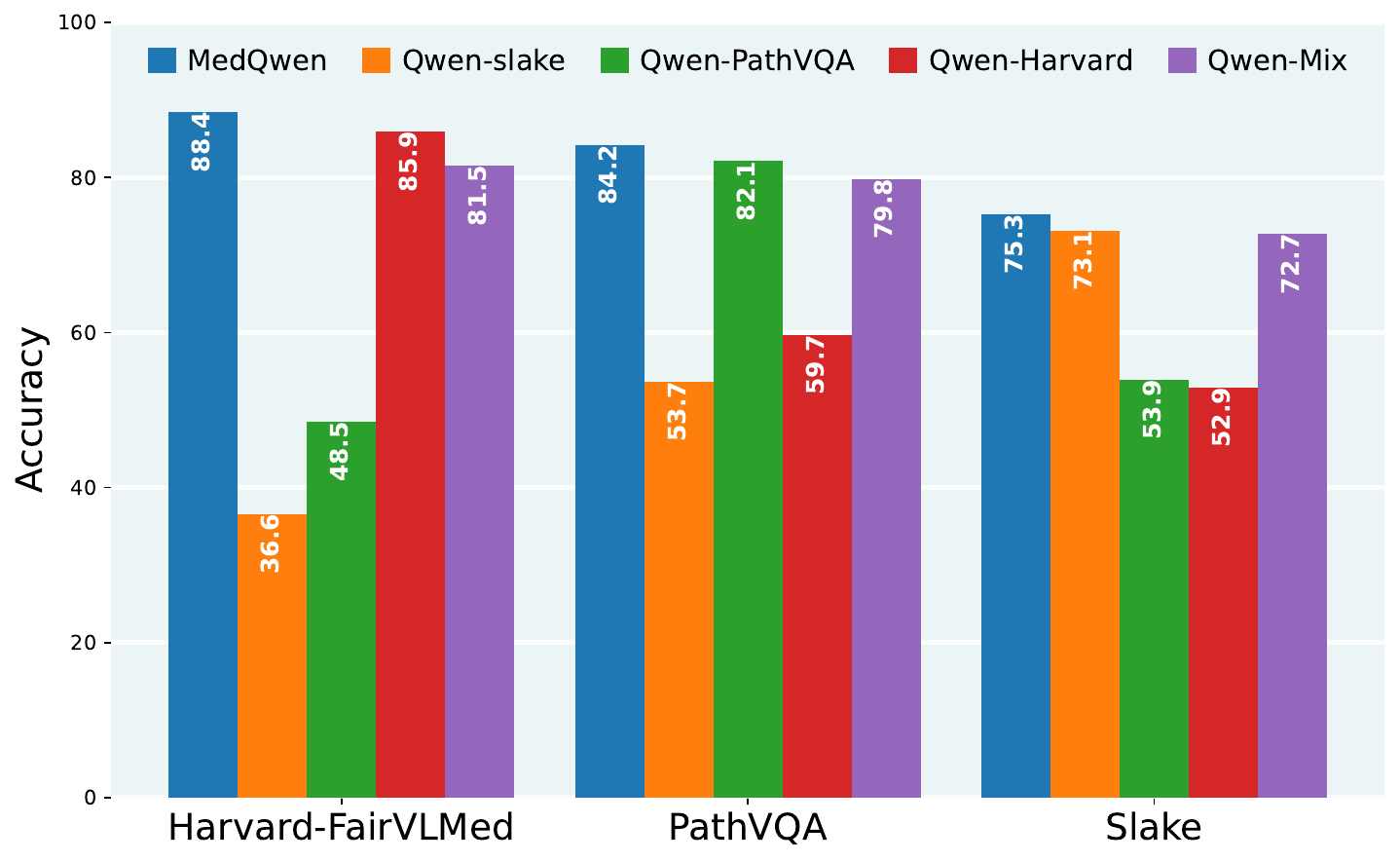}
    \caption{Model performances on three benchmarks when trained with different data configurations.}
    \label{fig:accuracy}
    \vspace{-10pt}
\end{figure}

Our investigations reveal that current VLMs fine-tuned using standard LoRA exhibit high sensitivity to the \emph{training data regime} \cite{li2024mixlora, chen2024llava}: models tuned on one medical dataset (e.g., Qwen-Slake in Fig.~\ref{fig:accuracy}) often regress on others, and naively mixing heterogeneous sources introduces cross-dataset interference (e.g., Qwen-Mix in Fig.~\ref{fig:accuracy}). This data-regime instability limits scalability across modalities, anatomies, and clinical tasks. Moreover, while tuning with broader data often reduces hallucinations~\cite{gunjal2024detecting, liu2024mitigating, wang2024vigc}, it could increase \emph{catastrophic forgetting} when tasks arrive sequentially, which is pervasive in clinical settings.


To address these limitations, we introduce \ours, a routed, \emph{SVD-structured LoRA Mixture of Experts (MoE)} that turns the singular-value spectrum of pre-trained weights to a set of non-overlapping spectral priors. Each expert is initialized from a distinct SVD segment, and a lightweight router activates only the few experts whose priors are most relevant for the current input. This design encourages specialization without discarding useful pre-training structure and directly targets cross-dataset interference. To stabilize optimization, we derive a theoretical scaling scheme that \emph{aligns low-rank updates with the gradient geometry of a fully fine-tuned MoE}, improving convergence without changing the base architecture or optimizer. 

Extensive experiments on 23 medical datasets spanning VQA, report generation, and classification validate the approach. MedQwen reaches SOTA accuracy among PEFT methods while preserving efficiency, approaches full fine-tuning quality, and exhibits robustness to sequential training: on a Harvard-FairVLMed to PathVQA protocol, a standard LoRA model loses over 50\% of its original accuracy, MoE-LoRA loses over 20\%, whereas MedQwen’s accuracy drops by only 5\%. We further analyze expert sparsity, rank, and routing, demonstrating favorable compute–performance trade-offs within a single-GPU budget. Our contributions are:
\begin{itemize}
\item \textbf{Sparse spectral LoRA}: We introduce an SVD-structured MoE that partitions pre-trained weights into non-overlapping spectral segments and routes inputs to the most relevant low-rank experts, reducing cross-dataset interference and preserving useful priors.

\item  \textbf{Optimization alignment with scaling}: We establish conditions under which LoRA-MoE matches full-rank MoE dynamics, provide a residual-matching initialization, and derive a theoretical scaling factor that preserves gradient geometry under low rank.

\item \textbf{Unified, efficient medical VLM}: Without architectural changes, \ours\ achieves SOTA results across a wide set of medical VQA, report generation, and classification tasks while remaining efficient.

\item \textbf{Comprehensive analysis}: We ablate expert count, activation sparsity, rank, and scaling, and study convergence and compute trade-offs.
\end{itemize}


%
%
\section{Related Work}
\label{sec:2}

\textbf{VLMs in Medicine.} Medical VLMs (Med-VLMs) have yielded excellent capabilities in image interpretation and VQA~\cite{nam2025multimodal, liu2025application}. XrayGPT~\cite{thawakar2024xraygpt} integrates a specialized visual encoder (MedClip)~\cite{wang2022medclip} with a fine-tuned VLM, employing a straightforward linear transformation layer to achieve alignment between visual and textual modalities. LLaVA-Med~\cite{li2023llava} further refines visual-textual alignment in medical contexts by curating high-quality image-text pairs from PubMed publications and synthetic VQA datasets. BiomedGPT~\cite{luo2024biomedgpt} utilizes a BERT-style encoder and a GPT-style decoder architecture, trained on multimodal datasets, outperforming much larger commercial models like Med-PaLM~\cite{singhal2025toward}.  HuatuoGPT-Vision~\cite{chen2024huatuogpt} scale supervision with PubMedVision dataset (1.3 million medical samples) to improve adaptability across tasks. Despite this progress, most Med-VLMs are optimized for conversational VQA/report generation and rely on PEFT, and broad versatility across discriminative and generative clinical tasks typically requires additional task-specific heads or training.

\noindent\textbf{Mixture of Experts.}
MoE models expand model capacity by activating only a subset of experts per token, enabling a sub-linear increase in computational cost in Transformers while retaining large representation power~\cite{jacobs1991adaptive, chen2023octavius, shazeer2017outrageously, zhang2025more}. Recent MoEs differ in their expert selection and routing strategies. LLaVA-MoLE \cite{chen2024llava} routes tokens to domain-specific experts within Transformer layers, effectively reducing data interference. Other MoE-based methods improve domain adaptation and lifelong learning, such as MoRAL \cite{yang2024moral}, LoRAMoE~\cite{dou2023loramoe}, and PESC~\cite{wu2024parameter}, while MoE-LoRA \cite{luo2024moelora} and MoCLE \cite{gou2023mixture} allocate or activate task-specific parameters based on layer or instruction clusters. Our work builds on this line by mixing LoRA experts across diverse medical datasets, highlighting MoE's potential to mitigate cross-dataset conflicts in multimodal medical settings.

\noindent\textbf{Parameter-Efficient Fine-Tuning.}
Given the large size of foundation models, recent research has focused on developing PEFT methods~\cite{hu2022lora, cheng2025revisiting, lester2021power, li2025ensembles}, which reduce fine-tuning costs by updating only a small subset of parameters.
PEFT approaches can be grouped into three main directions:
(1) Additive Rank/Scale: methods such as AdaLoRA~\cite{zhang2023adaloRA} and rsLoRA~\cite{wang2024roselora} dynamically allocate parameter budgets or adjust scaling factors;
(2) Architectural Improvements:  DoRA~\cite{liu2024dora} decomposes weights into magnitude and direction, while RoseLoRA~\cite{wang2024roselora} applies sparse low-rank adaptation;
and (3) Optimization and Initialization: including PiSSA \cite{meng2024pissa}, MiLoRA \cite{wang2025milora}, LoRA-Pro \cite{wang2024pro}, and LoRA-GA~\cite{wang2024ga}, which refine singular components or align low-rank gradients with full fine-tuning.
Despite these advances, a notable performance gap often remains between PEFT methods and full fine-tuning.

\noindent\textbf{SVD-based VLM Adaptation.} 
SVD is a classical tool for compression and latent semantic analysis \citep{horasan2019alternate, hsu2021language,tanwar2018dimensionality}, and has recently emerged as a promising approach for PEFT in VLMs~\citep{yuan2023asvd, wang2024svd, koleilat2025singular}. AdaLoRA \citep{zhang2022adaptive} and SARA \citep{gu2024sara} employ SVD to identify suitable ranks, improving parameter efficiency. PiSSA \citep{meng2024pissa} and MiLoRA \citep{wang2024milora} use SVD for LoRA initialization; PiSSA fine-tunes dominant components for faster convergence, while MiLoRA targets minor components to enhance task-specific adaptability.
In contrast, our approach (1) partitions spectral components into non-overlapping experts, (2) adaptively selects relevant priors via routing, and (3) corrects weight misalignment and gradient dynamics through theoretically grounded scaling.


\begin{figure*}
  \centering
        \includegraphics[width=\textwidth]{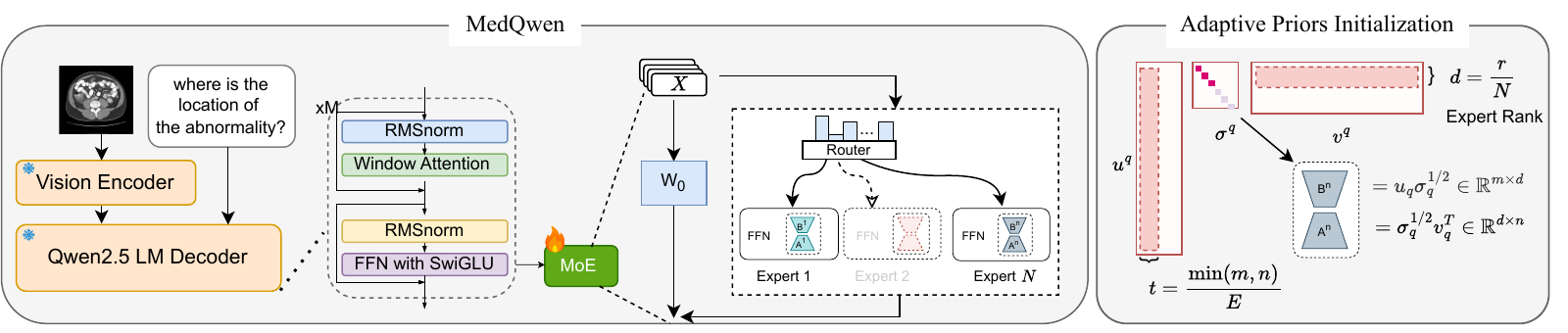}
        \captionof{figure}{An overview of the proposed \ours\ approach.}
    \label{fig:arch}
    \vspace{-1em}
\end{figure*}

%% file: sec/2_method.tex
\section{Background and Motivation}

\subsection{SVD Initialization Overview }   \label{sec:svd}

Initializing LoRA matrices with SVD is a common strategy to keep key structure from pre-trained weights~\cite{zhao2024galorememoryefficientllmtraining, meng2024pissa, wang2024kasaknowledge, lu2024twinmerging}. Existing methods differ mainly in which spectral components they update: PiSSA~\cite{meng2024pissa} fine-tunes the dominant singular directions, while MiLoRA~\cite{wang2024miloraharnessingminorsingular} targets the minor ones to improve task specificity.

To establish a unified theoretical framework that links SVD-based initialization methods with fine-tuning (FT), let the pre-trained weight matrix be $W^{(0)} \in \mathbb{R}^{m \times n}$, represented by its SVD $W^{(0)} = U S V^\top$. Assuming $h = \min(m, n)$ and a target LoRA rank $r$, we proceed to decompose $W^{(0)}$ into rank-$r$ blocks as follows:
\begin{align}
W^{(0)} = \sum_{i=0}^{l} U_i S_i V_i^\top,
\end{align}
Here, \( l = \frac{h}{r} - 1 \), and \( i \) refers to the segment \( [i \cdot r : (i+1) \cdot r] \).
For each segment, the submatrices are defined as
\( U_i = U_{[:,i \cdot r : (i+1) \cdot r]} \in \mathbb{R}^{m \times r} \),
\( S_i = S_{[i \cdot r : (i+1) \cdot r,\, i \cdot r : (i+1) \cdot r]} \in \mathbb{R}^{r \times r} \),
and \( V_i = V_{[:,i \cdot r : (i+1) \cdot r]} \in \mathbb{R}^{n \times r} \).
The subsequent FT procedures in different relevant techniques are expressed as:

\begin{equation}
\begin{aligned}
\text{Full FT}: &\quad U_0 S_0 V_0^\top + U_1 S_1 V_1^\top + \cdots + U_l S_l V_l^\top \\
\text{MiLoRA}: &\quad (U_0 S_0 V_0^\top + \cdots + U_{l-1} S_{l-1} V_{l-1}^\top)^* + U_l S_l V_l^\top \\
\text{PiSSA}: &\quad U_0 S_0 V_0^\top + (U_1 S_1 V_1^\top + \cdots + U_l S_l V_l^\top)^* \\
\text{KaSA}: &\quad (U_0 S_0 V_0^\top + \cdots + U_{l-1} S_{l-1} V_{l-1}^\top)^* + U^{\text{r}} S^{\text{r}} {V^{\text{r}}}^\top
\end{aligned}
\end{equation}
where \((\cdot)^*\) denotes frozen components. The trainable LoRA parameters are:
\begin{align}
B = U_i S_i^{1/2} \in \mathbb{R}^{m \times r}, \quad A = S_i^{1/2} V_i^\top \in \mathbb{R}^{r \times n}.\label{eq:ba}
\end{align}
We observe that PiSSA \cite{meng2024pissa} freezes the minor singular values and fine-tunes only the components \( U_0 S_0 V_0^\top \) with the largest norms, thereby achieving an optimal approximation to \( W^{(0)}\).
In contrast, MiLoRA \cite{wang2025milora} and KaSA \cite{wang2024kasaknowledge} retain segments \( 0 \sim (l-1) \) as preserved pre-trained knowledge, while KaSA regards the minor component \( U_l S_l V_l^\top \) as noise and replaces it with a new random component \( U^{\text{r}} S^{\text{r}} V^{\text{r}\top}\).
Empirically, PiSSA converges faster by emphasizing the principal singular values, whereas MiLoRA and KaSA maintain more pre-trained knowledge, leading to improved final performance. This observation highlights a key trade-off between focusing on principal components and emphasizing minor components.

According to \cite{goat}, the performance of FT varies across different singular value segments and depends on the dataset; for example, \( x = l \) yields superior performance on the Slake dataset, whereas \( x = 0 \) performs better on the PathVQA dataset.
Additionally, the intermediate segments contribute significantly to model performance.
When \( r = 128 \), the best results are typically achieved in the middle range of segments.
These results indicate that each singular value segment carries task-dependent information, thereby motivating an adaptive mechanism that enables the model to automatically select relevant segments during optimization while preserving the intrinsic structure of the pre-trained matrix.

\begin{figure}[t] 
    \centering
    \begin{subfigure}[t]{0.48\columnwidth}
        \centering
        \includegraphics[width=\linewidth]{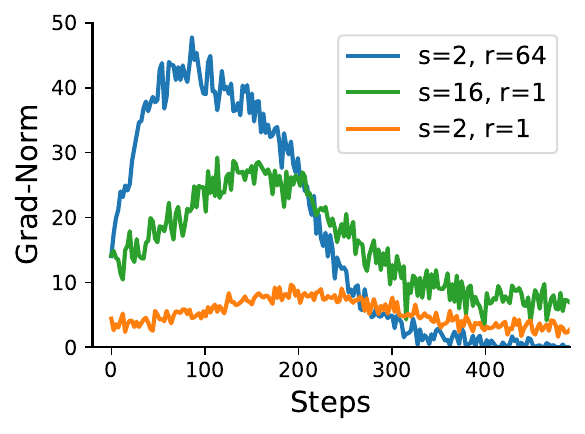}
        \caption{}
        \label{fig:grad_norm}
    \end{subfigure}
    \hfill
    \begin{subfigure}[t]{0.48\columnwidth}
        \centering
        \includegraphics[width=\linewidth]{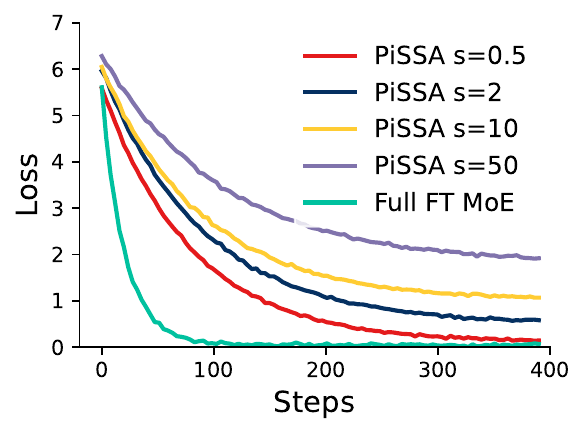}
        \caption{}
        \label{fig:pissa_loss}
    \end{subfigure}
    \vspace{-5pt}
    \caption{SVD initialization \vs scaling $s$ and rank $r$ \label{fig:pissa}.}
    \vspace{-10pt}
    \label{fig:training_comparison}
\end{figure}

\subsection{Scaling Factor Overview} \label{sec:scale}

In LoRA, the parameterization \( W = W^{(0)} + sBA \) is widely adopted, yet the role of the scaling factor \( s \) is often under-analyzed.
\citet{biderman2024lora} suggest setting \( s = 2 \),  the SVD-based approach of PiSSA \cite{meng2024pissa} empirically eliminates the dependency of \( sBA \) on \( s \) by dividing \( A \) and \( B \) by \( \sqrt{\frac{1}{s}} \), and \citet{tian2024hydraloraasymmetricloraarchitecture} report that employing larger scaling factors in LoRA-MoE architectures improves performance.

To analyze this effect, as depicted in Fig.~\ref{fig:pissa_loss}, we vary \( s \) in the SVD-based LoRA with a fixed rank and observe that \( s \) continues to influence the convergence rate.
To capture this effect quantitatively, we define:
\begin{equation}\label{def:eg}
    \tilde{W} = W + sBA, \qquad
    \tilde{g} = \frac{\partial L}{\partial \tilde{W}},
\end{equation}
where \( \tilde{W} \) and \( \tilde{g} \) denote the equivalent weight and gradient, respectively (see Fig. \ref{fig:optimization}). Also, $L$ denotes the loss function. At the optimization step \( t \), letting \( g_t \) be the full FT gradient and \( A \), \( B \) denote low-rank factors, the equivalent gradient is:
\begin{equation}
    \tilde{g}_t = s^2 \left( B_t {B_t}^\top g_t + g_t {A_t}^\top A_t \right).
\end{equation}
For SVD-based initialization, we obtain:
\begin{align}
\tilde{W} &\propto sBA = s \left( \frac{1}{s} U_{r} S_{r} V_{r}^\top \right) = U_{r} S_{r} V_{r}^\top, \\
\tilde{g} &= s^2 \left( \frac{1}{s} U_{r} S_{r} U_{r}^\top g + \frac{1}{s} g V_{r} S_{r} V_{r}^\top \right)  \notag\\
          &= s \left( U_{r} S_{r} U_{r}^\top g + g V_{r} S_{r} V_{r}^\top \right).
\end{align}

Thus, the equivalent weight is independent of \( s \), but the equivalent gradient scales linearly with it.
As illustrated in Fig.~\ref{fig:pissa}, a small scaling factor (\( s = 2 \)) results in slow convergence, whereas increasing \( s \) accelerates optimization. However, this trend does not extend indefinitely. Beyond a certain point, overly large $s$ values amplify the gradient too strongly, leading to unstable updates and degraded generalization. 
In practice, we observe that moderate scaling (e.g., $s \in [4, 16]$) achieves the best trade-off 
between convergence speed and stability.

When examining different ranks, we observe that with a small rank (e.g., \( r = 1 \)), the gradient norm diminishes, leading to a notable performance gap  relative to e.g., \( r = 64 \).
However, employing a higher scaling factor (\( s = 16 \)) restores the gradient norm and narrows the gap.
This effect is particularly advantageous in MoE configurations, where rank reduction across experts can be compensated by larger scaling factors, consistent with \citet{tian2024hydraloraasymmetricloraarchitecture}.


\section{Methodology}
\label{sec:methods}

\subsection{LoRA-MoE Architecture}
LoRA-MoE introduces \(N\) LoRA experts and employs a router to dynamically select and combine their outputs,  mitigating task interference by isolating task-specific adaptations while maintaining parameter efficiency. Let the parameters of the \(i\)-th LoRA expert be denoted by \(\{A_i, B_i\}\), and let \(R\) denote the router. Here, \(B_i \in \mathbb{R}^{m \times d}\) and \(A_i \in \mathbb{R}^{d \times n}\), where \(d = \frac{r}{N}\). The output of the MoE system can then be expressed as:
\begin{equation}\label{Eq:MOE}
    y = W^{(0)} x + \sum_{i=1}^N R(\mathbf{x})_i \left( s B_i A_i \mathbf{x} \right),
\end{equation}
where \(R(\mathbf{x})_i\) is the router score for the \(i\)-th expert, $W^{(0)}$ is the pre-trained weight matrix, and \(s\) is the LoRA scaling factor.

In dense gating, the router consists of a dense layer with trainable parameters \( W_g \).
The gating scores are computed using a softmax function:
\begin{equation}
     R(\mathbf{x})_i = \text{softmax}(W_g^\top \mathbf{x}).
\end{equation}
Dense gating is employed in the soft routing strategy, where the outputs of multiple experts are weighted and combined according to their gating scores.

We adopt the Mixtral top-\(k\) router~\cite{jiang2024mixtral}.
Specifically, an MoE layer comprises $N$ linear modules $\{W_{1}, \dots, W_{N}\}$ and a router $R_z \in \mathbb{R}^{m \times N}$ that assigns input $\mathbf{x}$ to experts based on routing scores:
\begin{equation}
    p^i(\mathbf{x}) = \frac{\exp(z^i(\mathbf{x}))}{\sum_{j=1}^{N} \exp(z^j(\mathbf{x}))},
\end{equation}
where $z(\mathbf{x}) = W_z \mathbf{x}$ and $p^i(\mathbf{x})$ is the score for expert $i$.
Let $S_k(x) \subset \{1,\ldots,N\}$ denote the indices of the top-$k$ experts
according to their gating logits $z(x)$, ensuring $|S_k(x)| = k$ and $z_i(x) > z_j(x)$ for all $i \in S_k(x)$ and $j \notin S_k(x)$.
The normalized top-$k$ gating weights are defined as:
\begin{equation}
R(\mathbf{x})_i =
\begin{cases}
\dfrac{\exp(z_i(x))}{
\sum_{j \in S_k(x)} \exp(z_j(x))}, & i \in S_k(x),\\[8pt]
0, & \text{otherwise.}
\end{cases}
\label{eq:mol_weights}
\end{equation}

The final LoRA MoE output combines the frozen base layer with the weighted contributions of the selected LoRA experts:
\begin{equation}
\mathrm{MoE}_{\text{LoRA}}(\mathbf{x}) =
W^{(0)} x +
\sum_{i \in S_k(x)}
R(\mathbf{x})_i \left( s B_i A_i \mathbf{x} \right).
\label{eq:mol_output}
\end{equation}

During training, only the experts indexed by $S_k(x)$ and their corresponding gating paths receive gradient updates. Since \( k \ll N \), LoRA-MoE activates significantly fewer parameters than a dense MoE.

\begin{figure}[t] 
    \centering
    \includegraphics[width=.47\textwidth]{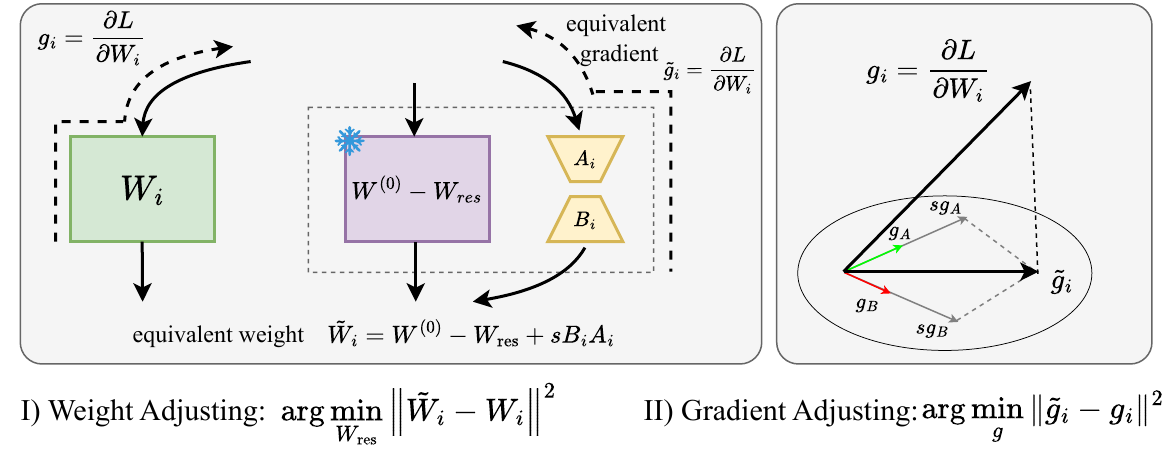}
    \caption{Optimization with SVD-structured MoE by separately aligning each expert. $W_{res}$ ensures the equivalent weight equals $W^{(0)}$ before optimization. Scaling aligns each expert’s equivalent gradients to those of full MoE FT.}
    \label{fig:optimization}
    \vspace{-15pt}
\end{figure}

\subsection{Adaptive Priors Initialization}

Building on the relevance of different SVD components on input data (Sec.~\ref{sec:svd}), we introduce a novel initialization methodology for a LoRA-MoE framework. This strategy assigns distinct, non-overlapping segments of a pre-trained model's SVD to individual experts. The core hypothesis is that \textit{the MoE's routing mechanism can learn to dynamically select the expert whose assigned singular value spectrum is most pertinent to a given input, creating a more adaptive low-rank approximation.}

Define a set of SVD component segments, $\mathcal{E}_r$, which are distributed evenly among the $N$ available experts:
\begin{equation}
    \mathcal{E}_r = \left\{(U_{[:,\, k:k + d]}, S_{[k:k+d,\,k:k+d]}, V_{[k:k+d,\,:]}^\top) \mid j = 1, \dots, N \right\},
\end{equation}
where $t = \frac{\min(m, n)}{N}$ and the starting index for the $j$-th expert is $k = (j - 1)t$. Each expert is allocated a uniform rank of $d = \frac{r}{N}$. Subsequently, each expert's constituent low-rank matrices, $B_i$ and $A_i$, are formed from a unique SVD segment $(U', S', V'^{\top}) \in \mathcal{E}_r$ as follows:
\begin{equation}
    B_i = \sqrt{\frac{1}{s}}\, U' S'^{1/2} \in \mathbb{R}^{m \times d}, \quad A_i = \sqrt{\frac{1}{s}}\, S'^{1/2} V'^{\top} \in \mathbb{R}^{d \times n}.
\end{equation}
The scaling factor $\sqrt{s}$ is applied to ensure that the resulting update matrix $sBA$ is invariant to $s$~\cite{meng2024pissa}. This structured initialization promotes expert specialization and enhances the model's flexibility for diverse FT tasks.

\subsection{Theoretical Optimization Alignment}
The direct integration of SVD spectral subspaces into MoE architectures introduces weight misalignment and complex gradient dynamics. These issues are not present in conventional LoRA that employs zero-initialization. This specific problem within MoE frameworks has been largely underexplored. We address this gap by first formally stating the necessary alignment requirements for stable training, and subsequently demonstrating that these requirements can be satisfied through a principled scaling mechanism.

\begin{theorem}[Single model alignment.]\label{thm:single-align}
Let \( \tilde{W}_t \) denote the effective LoRA weight and \( \tilde{g}_t \) its effective gradient, with \(W_t\) and \(g_t\) the corresponding quantities for full FT. Alignment follows from the pair of conditions:
\begin{equation}
\tilde{W}^{(0)} \approx W^{(0)},
\qquad
\tilde{g}_t \approx g_t,
\label{eq:single-align}
\end{equation}
\end{theorem}
(defined in Eq.~\ref{def:eg} as the ``equivalent weight'' and ``equivalent gradient''). These conditions mitigate the performance gap in single-adapter LoRA \cite{wanglora, wang2024loraprolowrankadaptersproperly}.

\begin{theorem}[MoE alignment.]\label{thm:MoE-ali}
For an MoE with experts \(i \in \{1,\dots,N\}\), router weights \(R(\mathbf{x})_i\) (top-\(k\) gating), and per-expert LoRA parameters, alignment reduces to enforcing
\begin{equation}
\tilde{W}^{(0)}_{\,i} \approx W^{(0)}_{\,i},
\qquad
\tilde{g}^t_{\,i} \approx g^t_{\,i},
\quad \forall i \in \{1,\dots,N\}.
\label{eq:moe-align}
\end{equation}
\end{theorem}
This makes our LoRA-MoE training dynamics equivalent to those of an upcycled MoE~\cite{he2024upcycling} with full-rank FT. This equivalence is highly desirable as it permits the optimization of each expert to proceed independently, simplifying the overall learning process and enhancing stability.

\noindent\textbf{Initialization alignment.}
We enforce this equivalence at initialization by adopting the principles of upcycled MoE \cite{he2024upcycling}. The procedure begins with the pre-trained base weight \(W^{(0)}\). For each expert, we initialize its low-rank factors \(B_i^{(0)} A_i^{(0)}\) using SVD priors. The resulting effective weight for the LoRA-MoE architecture is thus defined as:
\begin{equation}
\tilde{W}^{(0)}
= W^{(0)} - W_{\text{res}} + \sum_{i=1}^N R(\mathbf{x})_i\, s\, B_i^{(0)} A_i^{(0)}
\approx W^{(0)},
\label{eq:wres-approx}
\end{equation}
To ensure that the equivalent weight $\tilde{W}^{(0)}$ aligns with the pre-trained weight $W^{(0)}$ at initialization 
(see Eq.~\ref{eq:single-align}), we introduce a residual compensation term $W_{\text{res}}$ defined as the 
expected weighted contribution of all LoRA experts:
\begin{equation}
W_{\text{res}} = s \, \mathbb{E}_x \left[ \sum_{i=1}^{N} R(x)_i B_i^{(0)} A_i^{(0)} \right].
\label{eq:res-first}
\end{equation}
This term preserves the unbiased initialization of $\tilde{W}_0$. 

\begin{theorem}[Router moments.]\label{thm:router-mom}
For top-\(k\) gating over $N$ experts with input data denoted by $\mathbf{x}$, the router's behavior is simplified through the use of moment identities:
\begin{equation}
\mathbb{E}_{\mathbf{x}}\!\left[R(\mathbf{x})_i\right] = \frac{1}{N},
\qquad
\mathrm{Var}\!\left(R(\mathbf{x})_i\right) = \frac{N-k}{kN^{2}},
\label{eq:router-moments}
\end{equation}
\end{theorem}
for all $i\neq j \in \{1,\dots,N\}$. The variance of the initialization mismatch term, $W_{\text{res}}^{} - s \sum_{i=1}^N R(\mathbf{x})_i B_i^{(0)} A_i^{(0)}$, is proportional to the aggregate $\sum_{i=1}^N B_i^{(0)} A_i^{(0)}$. Given that in sparse gating scenarios (e.g., $2k<N$), the expert weight standard deviation, $\mathrm{std}(R(\mathbf{x})_i)$, constitute an important part of the mean $\mathbb{E}[R(\mathbf{x})_i]$, the effective control of this variance becomes critical.

\begin{theorem}[Residual matching objective.]\label{thm:residual}
To mitigate the initialization mismatch, we determine the optimal constant residual weight $W_{\text{res}}$ by minimizing the expected MSE:
\begin{equation}
W_{\text{res}}^{+}
= \arg\min_{W_{\text{res}}}
\ \mathbb{E}_{\mathbf{x}}
\left\|
W_{\text{res}} - s \sum_{i=1}^N R(\mathbf{x})_i B_i^{(0)} A_i^{(0)}
\right\|^{2}.
\label{eq:wres-opt}
\end{equation}
\end{theorem}
By substituting the established router moments from Eq. \ref{eq:router-moments}, we derive the following closed-form solution:
\begin{equation}
W_{\text{res}}^{+}
= \frac{s}{N} \sum_{i=1}^N B_i^{(0)} A_i^{(0)}.
\label{eq:wres-closed}
\end{equation}
Note that contemporary zero-initialization schemes for LoRA-MoE \cite{zadouri2024pushing,tian2024hydraloraasymmetricloraarchitecture} are recovered as the special case where $W_{\text{res}}^{+}$ vanishes.

\noindent\textbf{SVD-consistent scaling.}
To simultaneously maintain the information embedded in the SVD spectral subspaces in the pre-trained model and enforce numerical stability (i.e., by ensuring a tighter approximation in Eq. \ref{eq:wres-approx}), we damp per-expert low-rank initial factors by $\rho>0$:
\begin{equation}
B_i^{(0)} = \sqrt{\frac{1}{s\rho}}\, U_i S_i^{1/2},
\qquad
A_i^{(0)} = \sqrt{\frac{1}{s\rho}}\, S_i^{1/2} V_i^\top.
\label{eq:initAB}
\end{equation}
This joint scaling strategy effectively reduces the variance in initialization mismatch.

\input{tables/results}

\noindent\textbf{Gradient alignment.}
The theoretical analysis of the gradient dynamics under zero-initialized LoRA (where $B^{(0)}=0$, $A^{(0)} \sim U(-\sqrt{6/n},\sqrt{6/n})$) reveals that the effective LoRA gradient $\tilde{g}_i^t$ for expert $i$ is approximated by:
\begin{equation}
\tilde{g}_i^t
= s^{2}\Big( B_i^t {B_i^t}^{\!\top} g_i^t \;+\; g_i^t {A_i^t}^{\!\top} A_i^t \Big).
\label{eq:grad-surrogate}
\end{equation}
\begin{theorem} \label{thm:scale} The optimal scaling factor $s^{*}$, obtained by minimizing $\|\tilde{g}_t^i - g_t^i\|$ under a learning-rate ratio $\eta$ (Full FT vs. LoRA), is given by:
\begin{equation}
s^{*} = \sqrt{\frac{3n\,\eta}{r}}.
\label{eq:s-opt}
\end{equation}
\end{theorem}
Given the common practical condition $n \gg r$, the optimal factor $s^{*}$ is notably greater than two, providing a theoretical justification for why small scaling factors are insufficient and validating the necessity of moderate scaling (see Sec.~\ref{sec:scale}).

\noindent\textbf{Practical notes on SVD spectral subspaces}.
While the exact gradient dynamics for SVD-subspaces from the pre-trained model are hard to analyze, increasing $s$ and $\rho$ in Eq. \ref{eq:initAB} attenuates the initial magnitude of $B_i^{(0)}$ and $A_i^{(0)}$, pushing the system towards the zero-initialization regime where Eq. \ref{eq:s-opt} is valid. We adopt this scaling in \ours . 

%% file: tables/results.tex
\begin{table*}[!t]
\centering
\caption{Comparison of \ours\ with other VLMs and unified multi-modal models on medical visual comprehension tasks. \textbf{Bold} and \underline{underlined} text indicates the best performance and second-best performance, respectively. Close-ended: yes/no and other limited choices. Open-ended: Do not have a limited question structure and could have multiple correct answers.}
\resizebox{.9\textwidth}{!}{
\begin{tabular}{l|cc|ccccccc|c}
\toprule
\rowcolor[HTML]{E9F3FE}  &  &  & \multicolumn{2}{c}{\textbf{VQA-RAD \textuparrow}} & \multicolumn{2}{c}{\textbf{SLAKE \textuparrow}} & \multicolumn{2}{c}{\textbf{PathVQA \textuparrow}} &  {\textbf{OMVQA \textuparrow}}  &  \\
\cline{4-10}
\rowcolor[HTML]{E9F3FE} \multirow{-2}{*}{\textbf{Model}} & \multirow{-2}{*}{\textbf{\# Params}} & \multirow{-2}{*}{\makecell{\textbf{Medical} \\ \textbf{VLM}}} & \textbf{close} & \textbf{open} & \textbf{close} & \textbf{open} & \textbf{close} & \textbf{open} & \textbf{close}  & \multirow{-2}{*}{\textbf{Avg. \textuparrow}} \\
\midrule \midrule
 BLIP-2 & 6.7B & \Large \ding{55} & 43.4 & 20.5 & 41.6 & 32.1 & 48.5 & 23.8 &  26.9 & 33.8 \\
 LLaVA-v1.5 & 7B & \Large \ding{55} & 51.8 & 26.5 & 37.1 & 29.8 & 53.5 & 26.7 &  44.7 & 38.6 \\
 InstructBLIP & 7B & \Large \ding{55} & 61.0 & 28.2 & 66.8 & 40.7 & 56.0 & 27.3 & 29.0 & 44.1 \\
 Yi-VL & 6B & \Large \ding{55} & 52.6 & 27.1 & 52.4 & 30.8 & 54.9 & 25.7 &  50.2 & 41.9 \\
 InternVL2 & 8B & \Large \ding{55} & 64.9 & 33.8 & 66.6 & 35.2 & 60.0 & 34.6  & 54.5 & 49.9\\
 Llama-3.2 & 11B & \Large \ding{55} & 68.9 & 29.3 & \underline{72.4} & 37.1 & 62.8 & 36.5  & 63.2 & 52.9 \\
 Qwen-2.5-VL & 7B & \Large \ding{55} & 61.8 & 27.2 & 64.7 & 36.7 & 60.5 & 33.4  & 60.8 & 49.3 \\
 \midrule
 LLaVA-Med & 7B & \Large \ding{51} & 68.9 & 32.5 & 57.7 & 41.3 & 52.5 & 30.3 & 31.8 & 45.0 \\
 Med-Flamingo & 8.3B & \Large \ding{51} & 67.4 & 27.0 & 46.4 & 23.8 & 51.3 & 29.5 & 24.6  & 38.6 \\
 HuatuoGPT-Vision & 7B & \Large \ding{51} & 71.1 & 37.7  & 58.5 & 45.6 & 52.9 & 30.8 & 39.9  & 48.1 \\
 HealthGPT-M3 & 3.8B & \Large \ding{51} & 69.6 & 37.2 & 56.4 & 43.6 & 50.1 & 27.7 & 32.5 & 45.3 \\
 HealthGPT-L14 & 14B & \Large \ding{51} & \underline{74.5} & \underline{54.5} & 71.9 & \underline{56.2} & \underline{75.2} & \underline{42.1} & \underline{67.2}  & \underline{63.1} \\
 \midrule
 \ours & 7B & \Large \ding{51} & \textbf{78.8} & \textbf{59.6} & \textbf{75.3} & \textbf{59.9} & \textbf{84.2} & \textbf{49.1} & \textbf{70.6} & \textbf{68.2} \\
\bottomrule
\end{tabular}
}
\label{tab:results}
\end{table*}

%% file: sec/3_results.tex

\input{tables/app_vqa}

\input{tables/report_gen}

\input{tables/clip_result}

\section{Experiments}
We evaluate \ours\ along three axes: (1) accuracy gain on medical VQA compared to Med-VLM baselines, (2) clinical report generation on two distinct benchmarks, and (3) robustness across different training configurations and data regimes.

\subsection{Experimental Setups}
\label{sec:exp}
\textbf{Implementation Details}.
We select Qwen2.5-VL 7B as the base model for our architecture (Fig. \ref{fig:arch}).  \ours\ is initialized with the pre-trained weights of Qwen2.5 (7B). To ensure a fair comparison with prior work, we closely follow the configurations used in previous studies \cite{lin2025healthgpt, li2023llava}. Details of the baseline setups are provided in Appendix~\ref{app:implementation}.



\noindent\textbf{Datasets and Training Stages.}
Our model undergoes a three-stage training process. During the first and second phases, we utilize well-structured datasets provided by LLaVA-Med \cite{li2023llava} for \emph{alignment} and \emph{instruction tuning}. 
In the third phase, we employ medical datasets for \emph{MoE-tuning}, including SLAKE \cite{liu2021slake}, VQA-RAD \cite{lau2018dataset}, and PathVQA \cite{he2020pathvqa}, among others, covering a wide range of modalities and anatomical regions. Details of the datasets and their corresponding modalities are provided in Table \ref{tab:data_summary} and Table~\ref{tab:dataset_info}.



\subsection{Methods on Med-VQA}

On medical VQA, we compare \ours\ with both medical-specific and general-purpose LVLMs (Tables~\ref{tab:results} and \ref{tab:app_vqa}). As shown in Table~\ref{tab:results}, \ours\ achieves superior performance across nearly all metrics and datasets, surpassing the Qwen-2.5-VL baseline by an average of 18.9\% and outperforming Med-LLaVA and Med-Flamingo by 23.2\% and 29.6\%, respectively. To evaluate generalization, we also compare \ours\ with LVLMs trained on large-scale open-domain data (Table~\ref{tab:app_vqa}); \ours\ consistently outperforms these models, confirming its strong generalizability across diverse image domains and medical multimodal tasks.

\subsection{Medical report generation}
We also investigate the evaluation of LVLMs for the medical report generation task (Table~\ref{tab:report_gen}). In MIMIC-CXR, our method significantly outperforms LLaVA-Med and MedDr methods, achieving improvements of up to 23\% and 6.9\% in F1-RadGraph, and 28.6\% and 10\% in CheXbert, respectively. Furthermore, on IU-Xray, \ours\ surpasses RadFM, a model primarily focused on radiology tasks, across nearly all evaluated metrics. For instance, \ours\ achieves a BLEU-1 score of 41.59 and a CheXbert score of 64.93 on the IU-Xray dataset, underscoring its proficiency in comprehensively interpreting medical images. These empirical results demonstrate that the MoE effectively specializes to  diverse types of medical data while mitigating inter-dataset conflicts.

\subsection{Zero-Shot Classification}
Table \ref{tab:clip_result} reports zero-shot accuracies across nine radiology benchmarks spanning four imaging modalities. Our method achieves 95.31\% of the full FT MoE performance (58.8\% vs 61.7\%) despite using 339 times fewer parameters (2.24 vs. 760). It also outperforms the strongest single-LoRA baseline (rank-32, 55.45\%) by +3.38 points with 2.7 times fewer parameters (2.24\% vs. 5.98\%), and improves over PiSSA by 8.75\% and HydraLoRA by 4.84\%. Among LoRA-based methods (single and MoE), our approach achieves the best result on \emph{all nine} datasets and further outperforms full FT (56.76\%).
\input{tables/ablation}

\input{tables/param}

\begin{figure}[t] 
    \centering
    \includegraphics[width=.45\textwidth]{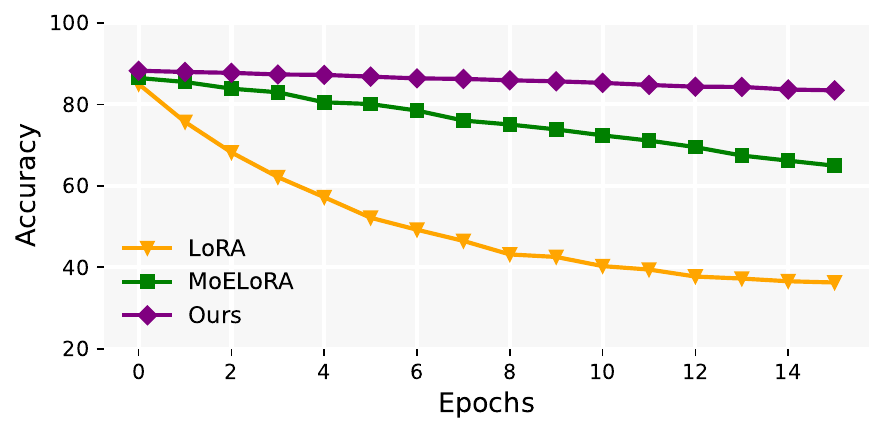}
    \vspace{-10pt}
    \caption{Comparison of catastrophic forgetting across different fine-tuning methods in sequential learning. The plot illustrates the accuracy decay over 15 epochs for LoRA, MoELoRA, and our method.}
    \label{fig:catastrophic}
    \vspace{-15pt}
\end{figure}

\subsection{Catastrophic Forgetting} \label{sec:Forgetting}
VLMs face a fundamental challenge when exposed to sequential or non-stationary learning scenarios—catastrophic forgetting \cite{kirkpatrick2017overcoming, zhai2024investigating, wang2025smolora}. This phenomenon refers to the abrupt and severe loss of previously acquired knowledge when a model is trained on new data or tasks. In other words, during sequential learning, new information tends to overwrite the representations formed from earlier experiences, leading to a substantial degradation in performance on previously learned tasks.

To examine the vulnerability of existing fine-tuning methods, such as LoRA and MoELoRA, to catastrophic forgetting during instruction tuning, we designed a simple yet effective evaluation framework. Specifically, we first fine-tuned the Qwen model using three different methods including our proposed MoE-SVD, standard LoRA, and MoELoRA on the Harvard-FairVLMed dataset. Subsequently, we fine-tuned these models on the PathVQA dataset, which belongs to a completely different medical domain, and evaluated their performance on the original Harvard-FairVLMed dataset to measure the degree of forgetting associated with each method.
As shown in Fig. \ref{fig:catastrophic}, the results reveal that the LoRA-based model suffered a drastic performance drop of over 50\%, while MoELoRA exhibited a moderate reduction of more than 20\%. In contrast, our proposed MoE-SVD approach demonstrated remarkable resistance to catastrophic forgetting, with only a 5\% decrease in accuracy after 15 epochs of sequential training. This finding underscores the effectiveness of our proposed method in preserving prior knowledge while maintaining adaptability across diverse domains. Detailed analyses are in Appendix \ref{app:cat}.

\subsection{Hallucination Evaluation}
We further include a medical hallucination evaluation in the Appendix \ref{app:hall}, which shows that \ours\ consistently outperforms both general-purpose and medical-specific LVLMs.

\subsection{Ablation Studies}
We ablate the impact of our adaptive priors initialization and gradient scaling  (Table \ref{tab:ablation}). Our initialization, with or without MoE scaling, consistently outperforms other methods (note that no SVD initialization corresponds to the original zero initialization, yielding 67.3/66.8). Without MoE, initializing a single LoRA with our SVD fragments achieves a performance of 60.1. In contrast, our MoE architecture achieves 67.3, demonstrating the advantages of expert specialization and the alignment strategy.


\begin{figure}[ht]
    \centering
    \begin{subfigure}[b]{0.48\columnwidth}
        \centering
        \includegraphics[width=\linewidth]{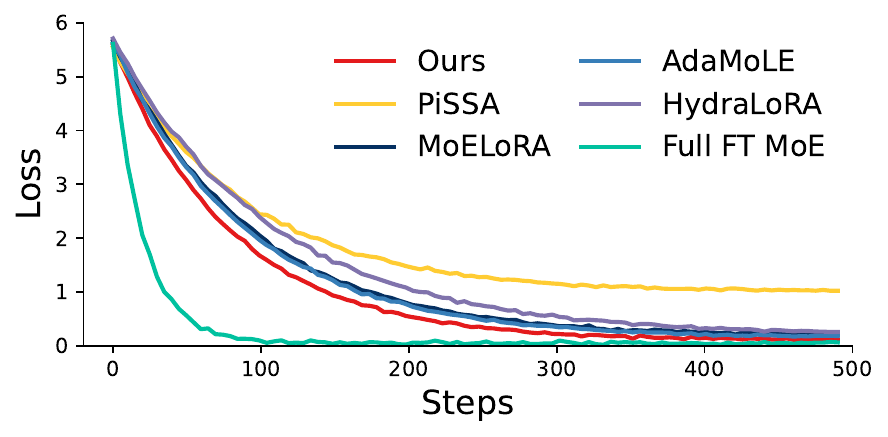}
        \caption{}
        \label{fig:loss}
    \end{subfigure}
    \hfill
    \begin{subfigure}[b]{0.48\columnwidth}
        \centering
        \includegraphics[width=\linewidth]{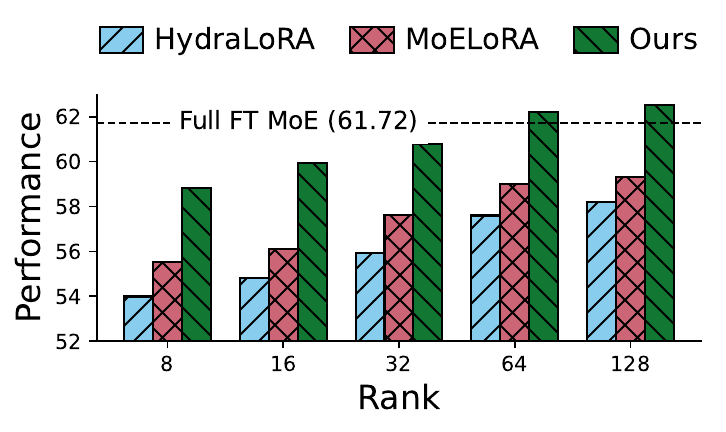}
        \caption{}
        \label{fig:r}
    \end{subfigure}
    \vspace{-5pt}
    \caption{(a) Training loss curves of different LoRA methods and Full Fine-tuning MoE on radiology datasets. {The balance loss is excluded in the MoE baselines for a fair comparison with single LoRA baselines}. (b) Performance of different methods across ranks.}
    \vspace{-15pt}
\end{figure}

\subsection{Convergence Speed}

As shown in Fig. \ref{fig:loss}, we compare the training loss curves of PiSSA, various LoRA MoE baselines, our proposed method, and Full FT MoE on radiology datasets. Our method demonstrates faster convergence than all LoRA MoE baselines and achieves performance closest to Full FT MoE. Notably, our method achieves a lower final loss, balancing performance and efficiency. In contrast, methods like PiSSA exhibit rapid initial drops but plateau at higher final loss (see Sec. \ref{sec:svd}).


\begin{figure}[ht]
    \vspace{-5pt}
    \centering
    \includegraphics[width=1\linewidth]{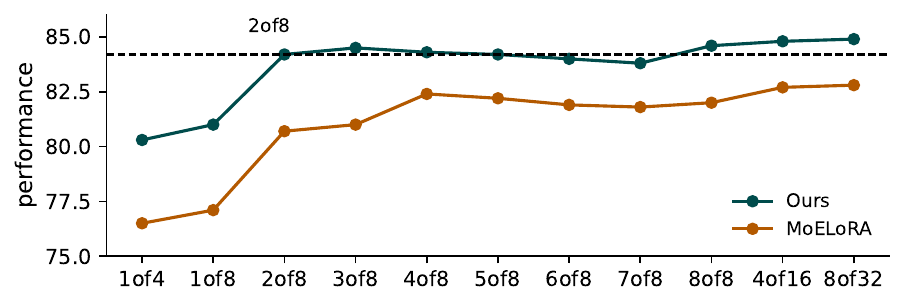}
    \vspace{-18pt}
\caption{Performance vs.~number of experts and activation ratio (total rank=32).\label{fig:expert_scale}}
    \vspace{-15pt}
\end{figure}

\subsection{Scalability}
\paragraph{Scaling with Rank.} To evaluate the scalability, we increase the rank in \ours\ from 8 to 128 on radiology datasets (Fig. \ref{fig:r}). As the rank increases, the performance gap between \ours\ and full fine-tuning MoE narrows significantly. Notably, \ours\ consistently outperforms both MoELoRA and HydraLoRA across all ranks. At rank 32, \ours\ achieves 60.78, surpassing MoELoRA (57.63) by 3.15\% and HydraLoRA (55.94) by 4.84\%. While higher ranks improve performance, gains diminish. For instance, \ours\ improves by 0.3\% from rank 64 to 128, indicating limited benefit despite higher computational costs.

\paragraph{Scaling with Expert Count and Activated Ratios.}
We also conduct experiments on medical VQA datasets with the total rank fixed at 32 (Fig.~\ref{fig:expert_scale}). Key findings include:
(1) With 8 experts, the 2-of-8 configuration achieves strong performance. Activating more experts may yield lower performance, showing that sparse expert activation is important.
(2) Increasing the total number of experts may improve performance, as seen in 2-of-8 vs. 4-of-16 / 8-of-32, but makes routers harder to train, increases memory consumption, and reduces runtime efficiency.
(3) Our method consistently outperforms MoELoRA, especially when activating only one expert, consistent with discussions in Sec. \ref{sec:scale}.
In practice, 2-of-8 offers a balanced trade-off between performance and storage efficiency.

\subsection{Computation Analysis}\label{sec:cost}
\paragraph{Parameter Size.}
The Params (\%) column (Table~\ref{tab:param}) compares parameter ratios of LoRA baselines and \ours. Our model achieves state-of-the-art performance with a parameter size of \(O(Hr) + O(He)\), much smaller than Full FT's \(O(H^2)\) and Full FT MoE's \(O(kH^2)\). Since \(r, e \ll H\), \ours\ is much more efficient. See Appendix \ref{app:parameters} for details.

\paragraph{FLOPs and Runtime}
We estimate the memory usage and performance of FT MoE based on single-GPU runtimes.
As shown in Table~\ref{tab:param}, the LoRA-MoE series trains much faster than Full FT MoE. Among LoRA-MoE variants, our method achieves the best performance with identical memory and time costs.
FLOPs analysis (Sec. \ref{app:flops}) reveals that Full FT MoE scales as \(O(ksH^2)\), while LoRA MoE simplifies to \(O(sH^2)\) since \(k < e\) and \(r \ll H\). Thus, LoRA MoE's FLOPs remain nearly constant, independent of \(k\), unlike Full FT MoE, which scales linearly with \(k\).

%% file: tables/app_vqa.tex
\begin{table}[htbp]
    \centering
    \footnotesize
    \caption{Accuracy (\%) of different Med-VLMs based on LLaVA-Med-1.5 on medical VQA task. }
    \vspace{-5pt}
    \resizebox{\columnwidth}{!}{
    \begin{tabular}{l|cc|c|cc}
    \toprule
        \rowcolor[HTML]{E9F3FE}  & \multicolumn{2}{c|}{\textbf{Radiology}} &\textbf{Ophthalmology} & \multicolumn{2}{c}{\textbf{Pathology}} \\
        \rowcolor[HTML]{E9F3FE} \multirow{-2}{*}{\textbf{Models}} & IU-Xray & MIMIC-CXR & Harvard-FairVLMed & Quilt-1M & PMC-OA \\ \midrule
        GPT-4o & 63.25 & 60.61 & 61.50 & 53.56 & 49.70 \\
        Gemini-1.5 & 59.73 & 61.02 & 58.53 & 56.88 & 52.17 \\
        LLaVA-v1.6	& 58.05 & 63.70 & 48.52 & 35.73 & 38.54 \\
        Qwen-2.5-VL & 59.43 & 60.43 & 38.06 & 28.74 & 29.53 \\
        InternVL-2 & 54.06 & 59.47 & 44.38 & 37.82 & 34.40 \\
        \midrule
        LLaVA-Med-1.5 & \underline{75.47} & \underline{75.79} & 63.03& \underline{62.80} & \underline{59.28}  \\
        Med-Flamingo & 26.74 & 61.27 & 42.06& 27.11& 32.62 \\
        MedVInT & 73.34 & 66.06& 35.92 & 26.81 & 27.77 \\
        RadFM & 26.67 &  69.30 &  52.47 & 27.02 & 25.12 \\
        miniGPT-Med & 54.87 & 53.92 & \underline{66.73} & 26.82 & 27.03 \\
        \midrule
        \ours\ &
        \textbf{90.33}& \textbf{84.68} & \textbf{88.43} & \textbf{73.74} & \textbf{66.32} \\
    \bottomrule
    \end{tabular}
    }
    \label{tab:app_vqa}
\end{table}

%% file: tables/report_gen.tex
\begin{table}[t]
\centering
\caption{
Model performance (\%) of different methods on report generation task.}
\vspace{-5pt}
\resizebox{\columnwidth}{!}{
\begin{tabular}{ll|ccccc|c}
\toprule
 & \multirow{2}{*}{Metric} & \multirow{2}{*}{RadFM} & LLaVA & Med      & \multirow{2}{*}{InternVL} & \multirow{2}{*}{MedDr} & \multirow{2}{*}{MedQwen}\\
 &                       &                        & Med   & Flamingo &                           &                        & \\
\midrule
\multirow{7}{*}{\rotatebox{90}{MIMIC-CXR}} &F1-RadGraph & 18.25 & 6.36 & 7.13  & 6.70  & \underline{22.44} & \textbf{29.40} \\
&BLEU-1      & 22.19 & 19.34 & 22.43 & 25.58 & \underline{27.28} & \textbf{33.74} \\
&BLEU-4      & 5.55  & 0.96  & 1.91  & 1.69  & \textbf{7.59} & \underline{7.38} \\
&ROUGE-1     & 28.88 & 21.44 & 21.69 & 22.87 & \underline{32.58} & \textbf{36.43}\\
&ROUGE-L     & 20.52 & 13.90 & 14.60 & 15.61 & \underline{22.59} & \textbf{25.69}\\
&CheXbert Vec & 31.18 & 15.56 & 18.69 & 15.90 & \underline{34.18} & \textbf{44.19}\\
&METEOR      & 20.42 & 13.25 & 14.02 & 17.54 & \underline{23.77} & \textbf{25.47}\\
\midrule
\multirow{7}{*}{\rotatebox{90}{IU-Xray}} &F1-RadGraph & 29.17 & 4.72  & 11.91 & 11.21 & \underline{33.19} & \textbf{37.76}\\
&BLEU-1      & \underline{40.78} & 14.85 & 16.28 & 18.27 & 37.71 & \textbf{41.59}\\
&BLEU-4      & 10.28 & 0.68  & 2.06  & 1.73  & \textbf{12.22} & \underline{11.61}\\
&ROUGE-1     & 36.79 & 12.98 & 14.93 & 18.22 & \underline{39.49} & \textbf{42.06}\\
&ROUGE-L     & 26.07 & 9.96  & 11.11 & 13.43 & \underline{28.35} & \textbf{29.65}\\
&CheXbert Vec & \underline{59.26} & 14.01 & 24.25 & 24.17 & 56.47 & \textbf{64.93}\\
&METEOR      & 30.92 & 14.48 & 17.69 & 20.03 & \underline{32.34} & \textbf{33.38}\\
\bottomrule
\end{tabular}}
\label{tab:report_gen}
\end{table}

%% file: tables/clip_result.tex
\begin{table*}
    \centering
    \caption{Zero-shot classification evaluation of BiomedCLIP ViT-B/16 with full fine-tuning and LoRA variants (rank is 8 if not specified) across various radiology datasets. \vspace{-5pt}}
    \label{tab:clip_result}
    \resizebox{\textwidth}{!}{
    \begin{tabular}{lccccccccccc}
        \toprule
        \multirow{2}{*}{\textbf{Approach}} & \multirow{2}{*}{\textbf{\# Params (\%)}} & \multicolumn{2}{c}{\textbf{X-ray}} & \multicolumn{2}{c}{\textbf{Ultrasound}} & \multicolumn{2}{c}{\textbf{MRI}} & \multicolumn{3}{c}{\textbf{CT}} & \multirow{2}{*}{\textbf{Avg.}} \\
        \cmidrule(lr){3-4} \cmidrule(lr){5-6} \cmidrule(lr){7-8} \cmidrule(lr){9-11}
        & &  \textbf{CheXpert(5x200)} & \textbf{RSNA} & \textbf{Thyroid} & \textbf{Breast} & \textbf{ACL} & \textbf{Meniscus} & \textbf{Axial} & \textbf{Coronal} & \textbf{Sagittal} & \\
        \midrule
        Full FT  &    100 & 65.30 & 72.56 & 76.05 & 76.12 & 85.89 & 40.09 & 40.13 & 28.09 & 26.59 & 56.76\\
        Full FT MoE & 760 & 72.39 & 74.53 & 78.59 & 88.50 & 87.38 & 46.34 & 45.09 & 33.40 & 29.27 & 61.72 \\
        \midrule
        \textit{Single LoRA Methods} \\
        \midrule
        LoRA &                     1.49 & 54.35 & 69.81 & 62.72 & 72.13 & 79.03 & 33.89 & 35.23 & 27.41 & 21.69 & 50.70\\
        LoRA (rank16) &            2.99 & 61.07 & 72.37 & 63.28 & 74.98 & 82.71 & 37.94 & 42.79 & 30.45 & 26.84 & 54.71\\
        LoRA (rank32) &            5.98 & 61.88 & 73.17 & 64.13 & 75.87 & 83.65 & 39.55 & 42.85 & 30.55 & 27.44 & 55.45 \\
        MiLoRA  &                  1.49 & 52.68 & 69.19 & 61.52 & 71.92 & 78.38 & 33.68 & 34.13 & 26.07 & 20.77 & 49.82 \\
        PiSSA &                    1.49 & 53.41 & 68.62 & 63.11 & 71.48 & 77.52 & 32.92 & 34.49 & 27.68 & 21.50 & 50.08\\
        \midrule
        \textit{LoRA MoE Methods} \\
        \midrule
        MoELoRA &                  2.24 & 62.36 & 72.80 & 64.85 & 74.89 & 83.83 & 40.53 & 41.91 & 30.84 & 27.65 & 55.52\\
        AdaMoLE &                  2.33 & 60.10 & 72.35 & 63.17 & 75.38 & 82.69 & 39.15 & 41.70 & 30.42 & 26.51 & 54.61\\
        HydraLoRA &                1.58 & 61.77 & 71.56 & 64.24 & 75.08 & 82.51 & 38.73 & 38.26 & 29.55 & 24.20 & 53.99\\
        \rowcolor{gray!20} Ours  & 2.24 & \textbf{66.83} & \textbf{74.14} & \textbf{77.88} & \textbf{77.23} & \textbf{86.57} & \textbf{42.45} & \textbf{44.05} & \textbf{32.10} & \textbf{28.24} & \textbf{58.83} \\
        \bottomrule
    \end{tabular}}
    \vspace{-15pt}
\end{table*}

%% file: tables/ablation.tex
\begin{table}[ht]
\scriptsize
\vspace{-5pt}
\centering
\caption{Ablation study of \ours. MoE denotes using the MoE architecture instead of a single LoRA. MS refers to using MoE scaling. O, P, M, and R represent different SVD-based initializations: original \ours\ segments, segments containing the principal singular values, segments containing the minor singular values, and randomly selected segments, respectively.}
\resizebox{0.85\linewidth}{!}{%
\begin{tabular}{lllllcc}
\toprule
\multirow{2}{*}{\textbf{MoE}} & \multicolumn{4}{c}{\textbf{SVD Initialization}} & \multirow{2}{*}{\textbf{Avg.}}& \multirow{2}{*}{\textbf{Avg. (w/o MS)}} \\ \cline{2-5}
                       & \textbf{P} & \textbf{M} & \textbf{R} & \textbf{O} &                                                 \\ \midrule
 \multicolumn{1}{c}{\omark}                    &           &           &       &                              & \multicolumn{1}{c}{67.3}&66.8                 \\
\multicolumn{1}{c}{\omark}                    &     \multicolumn{1}{c}{\omark}      &          &       &                                   & \multicolumn{1}{c}{67.4}&66.6              \\
\multicolumn{1}{c}{\omark}                    &           &     \multicolumn{1}{c}{\omark}      &      &                                   & \multicolumn{1}{c}{67.6}&66.9                 \\
\multicolumn{1}{c}{\omark}                    &           &           &    \multicolumn{1}{c}{\omark}    &                               & \multicolumn{1}{c}{67.7}& 66.9               \\
                     &         &           &       &  \multicolumn{1}{c}{\omark}   & \multicolumn{1}{c}{/}&60.1                 \\
\rowcolor[HTML]{DAE0FB}\multicolumn{1}{c}{\omark}                    &         &           &       &   \multicolumn{1}{c}{\omark}    & \multicolumn{1}{c}{\textbf{68.2}} & \textbf{67.3}             \\
 \bottomrule
\end{tabular}
}
\vspace{-4mm}
\label{tab:ablation}
\end{table}

%% file: tables/param.tex
\begin{table*}[t]
    \centering
    \caption{Comparison of our method against MoELoRA and HydraLoRA in memory
cost, training time, and performance. Memory cost was
measured and training time was recorded on the medical VQA benchmarks
using one A100 GPU with identical batch sizes.}
\vspace{-5pt}
    \resizebox{\textwidth}{!}{
    \begin{tabular}{ll|ccccccc|ccc}
        \toprule
        \multicolumn{2}{c|}{Model} & \multicolumn{2}{c}{\textbf{VQA-RAD}} & \multicolumn{2}{c}{\textbf{SLAKE}} & \multicolumn{2}{c}{\textbf{PathVQA}} & \multirow{2}{*}{\textbf{OMVQA}}& \multirow{2}{*}{\textbf{Params(\%)}} & \multirow{2}{*}{\textbf{Memory Usage}} & \multirow{2}{*}{\textbf{Epoch Time}}  \\
        \multicolumn{2}{c|}{\multirow{-2}{*}{\textbf{}}} & \textbf{close} & \textbf{open} & \textbf{close} & \textbf{open} & \textbf{close} & \textbf{open}    &  &   \\
        \midrule \midrule
        \multirow{3}{*}{\ours\ w/}
        & +MoELoRA & 72.3 & 57.2 & \underline{70.5} & \underline{53.7} & \underline{78.4} & \underline{38.6}  & \underline{65.1} & 0.96 & 32.56 GB & 29h15min \\
        & +HydraLoRA & \underline{73.6} & \underline{57.4} & 66.4 & 52.4 & 74.2 & 36.0 & 64.9 & 0.84 & 32.47 GB & 29h17min \\
        & \cellcolor[HTML]{DAE0FB}+Ours & \cellcolor[HTML]{DAE0FB}\textbf{78.8} & \cellcolor[HTML]{DAE0FB}\textbf{59.6} & \cellcolor[HTML]{DAE0FB}\textbf{75.3} & \cellcolor[HTML]{DAE0FB}\textbf{59.9} & \cellcolor[HTML]{DAE0FB}\textbf{84.2} & \cellcolor[HTML]{DAE0FB}\textbf{49.1} & \cellcolor[HTML]{DAE0FB}\textbf{70.6} & \cellcolor[HTML]{DAE0FB}\textbf{0.96} & \cellcolor[HTML]{DAE0FB}{32.56 GB} & \cellcolor[HTML]{DAE0FB}{29h19min}  \\
        \bottomrule
    \end{tabular}
    }
\label{tab:param}
\vspace{-10pt}
\end{table*}

%% file: sec/4_conclusion.tex
\section{Conclusion}
We propose \ours, a novel large medical vision-language model that unifies comprehension and generation through an enhanced LoRA fine-tuning approach. \ours\ adaptively integrates SVD-structured priors and aligns low-rank gradients with the fully fine-tuned MoE via theoretical scaling. Without modifying the underlying architecture or training algorithms, \ours\ substantially improves both efficiency and performance, achieving state-of-the-art results across 23 diverse medical comprehension and generation tasks. These results demonstrate its strong potential for real-world healthcare applications. 

\section*{Acknowledgements}
Funded by the Natural Sciences and Engineering Research Council of Canada (NSERC) and the Government of Canada’s New Frontiers in Research Fund (NFRF), [NFRFE-2022-00295].

%% file: sec/X_suppl.tex
\clearpage
\renewcommand{\thetable}{S\arabic{table}}
\renewcommand{\thefigure}{S\arabic{figure}}
\setcounter{table}{0}
\setcounter{figure}{0}

\maketitlesupplementary

\section{Pseudocode}
The pseudocode is provided below.
\begin{algorithm}[H]
\caption{\ours}\label{alg:goat}
\begin{algorithmic}[1]
    \Require Input vector $x$, input dimension $n$, hyperparameters $\eta,\rho$, number of experts $N$
    \Ensure Output $y = \tilde W^{(0)}x + \sum_{i=1}^N R(\mathbf{x})_i\, s\, B_i^{(0)} A_i^{(0)}$ x

    \Procedure{Initialization}{}
        \State \textbf{Scaling factor:} $s \gets \sqrt{3n\eta/r}$
        \State \textbf{SVD decomposition:}$W^{(0)} = U\,S\,V^\top$
        \For{$i = 1$ to $N$}
            \State $B_i^{(0)} \gets \sqrt{1/(s\rho)}\, U' S'^{1/2}$
            \State $A_i^{(0)} \gets \sqrt{1/(s\rho)}\, S'^{1/2} V'^\top$
        \EndFor
        \State $W_{\text{res}}^+ \gets \frac{s}{N} \sum_{i=1}^N B_i^{(0)} A_i^{(0)}$
        \State $\tilde{W}^{(0)} \gets W^{(0)} - W_{\text{res}}^+$
        \State \Return $\tilde W^{(0)},\{B_i^{(0)},A_i^{(0)}\}$
    \EndProcedure

    \Procedure{Forward}{$x$}
        \State Compute gating weights $R(\mathbf{x})_i$
        \State \Return $\tilde W^{(0)}x + \sum_{i=1}^N R(\mathbf{x})_i\, s\, B_i^{(0)} A_i^{(0)}$ x
    \EndProcedure
\end{algorithmic}
\end{algorithm}

\section{Experiment Details}
\subsection{Dataset Information}
Tables~\ref{tab:data_summary} and \ref{tab:dataset_info} summarize the datasets used in this study, covering a wide range of biomedical imaging modalities, such as MRI, CT, ultrasound, X-ray, and others. Each dataset is described in terms of its imaging modality, number of images, and question–answer text.

\noindent\textbf{Visual Question Answering:} VQA-RAD~\cite{lau2018dataset} contains 3,515 question–answer (QA) pairs and 315 radiology images, with questions spanning 11 categories and including both closed-ended and open-ended types. SLAKE~\cite{liu2021slake} comprises 642 radiology images and over 7,000 QA pairs, along with segmentation masks and object detection bounding boxes. PathVQA~\cite{he2020pathvqa} includes 4,998 pathology images with 32,799 QA pairs focusing on attributes such as location, shape, color, and appearance, categorized into open-ended and closed-ended types.
OmniMedVQA~\cite{hu2024omnimedvqa} comprises 118,010 medical images and 127,995 QA pairs collected from 73 different medical datasets, encompassing 12 imaging modalities and covering more than 20 distinct anatomical regions. Importantly, all images in this benchmark originate from authentic clinical scenarios, ensuring alignment with real-world medical requirements.

\noindent\textbf{Report Generation:} MIMIC-CXR~\cite{johnson2019mimic} includes 371,920 chest X-rays associated with 227,943 imaging studies from 65,079 patients. Following RadFM~\cite{wu2025towards} and R2Gen~\cite{chen2020generating}, we use 337,292 cases for training. IU-Xray~\cite{demner2016preparing} consists of 7,470 chest X-ray images paired with corresponding diagnostic reports; following R2Gen~\cite{chen2020generating}, we use 4,730 cases from the training split.

\noindent\textbf{Classification:} We use the UniMed \cite{khattak2024unimed} dataset for pretraining and evaluate our model on standard zero-shot classification benchmarks commonly used for medical VLM evaluation, covering four imaging modalities: X-ray, MRI, CT, and ultrasound.
For evaluation, we use the test sets from these widely recognized medical VQA datasets and additionally assess classification performance.

\input{tables/data_summary}

\subsection{Implementation Details and Hyperparameters}\label{app:implementation}

Visual question answering and image classification experiments are conducted on a single NVIDIA A100 GPU with 40~GB of RAM. Additional ablation and report generation experiments are performed on four A100 GPUs, each with 40~GB of RAM. All models are trained and evaluated using bfloat16 precision.

We fine-tune our model on each task using carefully selected hyperparameters to ensure optimal performance. Detailed configurations, including learning rate, batch size, number of epochs, and other training settings, are provided to ensure reproducibility and consistency across experiments. These hyperparameters are summarized in Table~\ref{tab:vqa_hyper} and Table~\ref{tab:cl_hyper}.
We set $\rho = 10$. The ratio between the full fine-tuning learning rate and the LoRA learning rate ($\eta$) is empirically set to 1 for ViT. In the Qwen experiments, when using a learning rate at the $1 \times 10^{-4}$ level, we set $\eta = 0.1$; when using a learning rate at the $1 \times 10^{-5}$ level, we set $\eta = 1$. This configuration follows common practice, where LoRA-based tuning typically employs a learning rate around $1 \times 10^{-4}$, while full fine-tuning methods operate at a lower rate near $1 \times 10^{-5}$.
For LoRA-MoE experiments, we set the balance loss coefficient to $1 \times 10^{-3}$. We adopt a top-$k$ routing strategy with $k=2$, which outperforms other routing strategies as shown in Fig.~\ref{fig:expert_scale}. The same routing strategy is applied consistently across all LoRA-MoE baselines.

\input{tables/datasets}
\input{tables/hyperparameter}

\subsection{Evaluation Metrics}
Following~\citep{xia2024rule,lin2023medical}, we use Accuracy, F1 Score and AUROC for evaluating the medical VQA task, and BLEU Score~\citep{papineni2002bleu}, ROUGE-L~\citep{lin2004rouge}, and
METEOR~\citep{banerjee2005meteor} for evaluating the report generation task. In alignment with existing hallucination benchmarks in both general and medical domains, accuracy (Acc) is employed as
the primary metric for evaluating close-ended hallucination. Assessing open-ended hallucinations in generated reports, we follow CheXpert~\cite{irvin2019chexpert} and measure hallucination rates
using CHAIR~\cite{li-etal-2023-evaluating}, which evaluates key symptom-centered visual findings. CHAIR is defined as:
$
\text{CHAIR} = \frac{|\mathcal{G} - \mathcal{S}|}{|\mathcal{G}|},
$
where $\mathcal{G}$  represents the set of findings extracted from the generated report using CheXbert~\cite{smit2020combining}, and $\mathcal{S}$ represents the set of findings extracted from the
real report using the same method.
For a more comprehensive evaluation, we additionally report key findings recall (Recall) and assess overall report quality using specialized metrics such as CheXbert~\cite{smit2020combining},
RadGraph~\cite{jain2021radgraph}, and RaTEScore~\cite{zhao2024ratescore}, which have been specifically developed for medical report generation. These metrics align closely with radiologists' assessments,
making them particularly suitable for evaluating the generation of open-ended medical reports, as demonstrated by RaTEScore data.

We evaluate the effectiveness of \ours\ in mitigating hallucinations in medical LVLMs across three medical benchmarks \cite{chang2025medheval}:

\begin{itemize}
    \item \textbf{Visual misinterpretation hallucination:} This category is evaluated using two datasets — Multi-Modality Visual Hallucination (MM-VisHal) and Chest X-ray Visual Hallucination (CXR-VisHal).

    \item \textbf{Knowledge deficiency hallucination:} This dataset is constructed from the MIMIC-CXR test set, where imaging reports are used as interpretations to prompt GPT-4 for generating diagnostic questions.

    \item \textbf{Context misalignment hallucination:} This benchmark links MIMIC-CXR data with the de-identified MIMIC-IV-EHR dataset \cite{johnson2023mimic} via subject IDs, providing comprehensive medical notes corresponding to each chest X-ray.
\end{itemize}

\begin{figure}[htbp]
    \centering
    \includegraphics[width=.95\columnwidth]{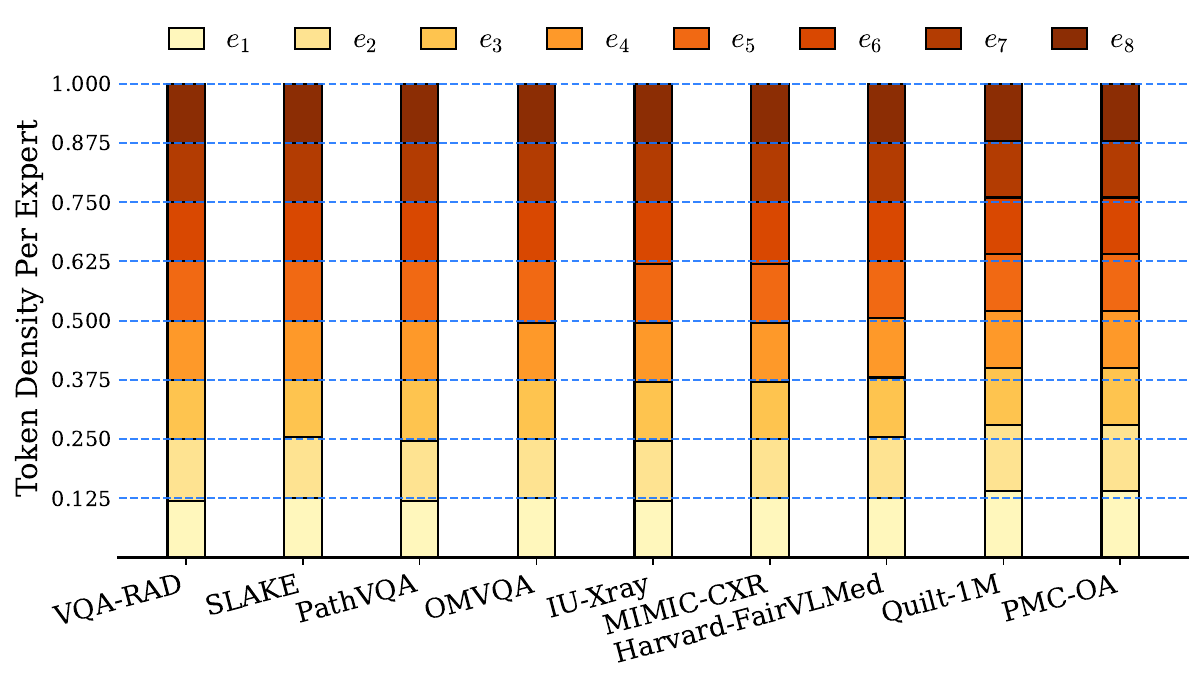}
    \caption{Expert load distribution across medical datasets.}
    \label{fig:rout}
\end{figure}

\begin{figure}[h]
\centering
\begin{subfigure}{0.49\columnwidth}
    \centering
    \includegraphics[width=\linewidth]{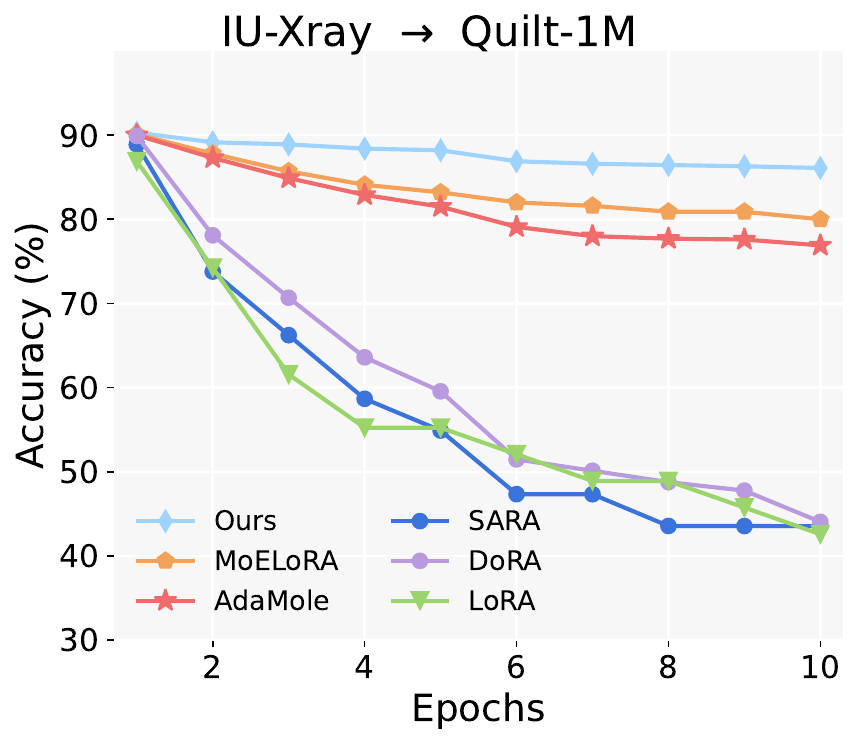}
\end{subfigure}
\hfill
\begin{subfigure}{0.49\columnwidth}
    \centering
    \includegraphics[width=\linewidth]{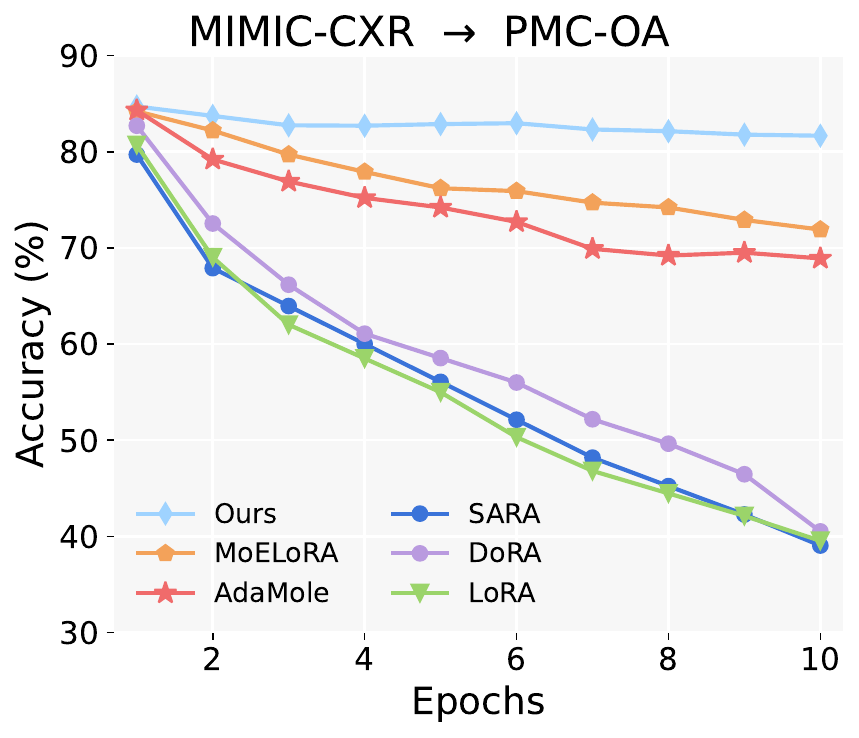}
\end{subfigure}
\caption{Catastrophic forgetting results.}
\label{fig:combined}
\end{figure}

\section{Additional Details}

\subsection{Routing Analysis}
Fig.~\ref{fig:rout} reports the expert load during training on nine datasets. With eight experts and two activated per token, the expected token density per expert is 0.125. The observed routing remains well balanced: no expert becomes inactive, and the load fluctuates around 0.125 within approximately $\pm$ 10\%. Because each expert is initialized from a distinct SVD-derived subspace, this stable and non-degenerate routing suggests that the spectral priors remain functionally distinct and continue to guide specialization, rather than behaving like indistinguishable zero-initialized adapters, even after variance-reducing damping (via $\rho$ and scaling). If the damping had erased the spectral structure, we would expect pronounced routing skew or expert collapse (i.e., a subset of experts dominating), a phenomenon commonly reported in prior MoE analyses.

\input{tables/rout}
\subsection{Catastrophic Forgetting}\label{app:cat}
To further evaluate the robustness of our method against forgetting, we extend the continual-learning setting to four datasets and multiple task sequences, as illustrated in Fig. \ref{fig:combined}. This expanded evaluation provides a broader view of retention performance across different sequential learning scenarios. We additionally include DoRA \cite{liu2024dora} and SARA \cite{hu2024sara} as representative SVD-based baselines for comparison. Across all dataset sequences, our MoE-SVD consistently maintains performance degradation below 5\%, indicating stable knowledge retention even when the model is exposed to multiple domain shifts.

\subsection{Routing Techniques}
To further investigate the effect of routing strategies, we conduct additional experiments comparing several alternatives, including top-$p$ routing and a top-$k$ routing variant with shared experts (Table~\ref{tab:routing_comparison}). Among the evaluated configurations, the top-$k$ strategy with $k=2$ consistently provides the strongest performance, outperforming the other routing schemes.

\input{tables/qwen_pro}

\subsection{Proper Scaling}
MedQwen assumes that the LoRA-MoE adapters are properly scaled at initialization. 
However, this assumption may not hold in practice. To address this issue, we extend the formulation to unscaled settings by aligning the scaling factors of the experts. 
In particular, we treat the first expert, which corresponds to the dominant low-rank spectral component and is analogous to the first expert in full MoE fine-tuning, as the reference with scaling factor $s_1$. 
To reduce the gap between our method and full MoE fine-tuning, we align the scaling factors $s_i$ of the remaining experts so that their effective magnitudes match the reference expert. 
Formally, the scaling factors must satisfy:
$s_1^2 \sigma_0 = s_i^2 \sigma_i,$
which yields: $s_i = s_1 \sqrt{\frac{\sigma_0}{\sigma_i}}.$
Here, $\sigma_i$ denotes the spectral mass of the corresponding segment, computed as the sum of singular values when the segment rank exceeds one. 
Importantly, only the scaling factors $s_{i>1}$ are adjusted, while the adapter initialization itself follows Eq.~\ref{eq:initAB}. 
We refer to this variant as \textbf{MedQwen-e}. 
Empirically, MedQwen-e achieves performance comparable to MedQwen across benchmarks (Table~\ref{tab:qwen_pro}), indicating that the proposed scaling strategy stabilizes the method even in unscaled initialization settings.

\section{Hallucination Evaluation}\label{app:hall}
\subsection{Visual Misinterpretation Hallucination}
A visual misinterpretation hallucination occurs when the model interprets fundamental visual components that are factually incorrect
or unsupported by medical evidence.

\noindent
\textbf{Hallucination Evaluation on Closed-Ended Evaluations. } 
The close-ended evaluation results on MM-VisHal and CXR-VisHal are detailed in Table~\ref{tab:close_ended_VMH}. Our method, \ours, significantly outperforms existing state-of-the-art models across all metrics and benchmarks, indicating the effectiveness of MoE in reducing medical hallucinations in close-ended medical tasks. GPT-4o, owing to its large-scale training and strong cross-domain instruction following, generally exhibits superior resistance to hallucinations compared to other (Med)-LVLMs. In contrast, general-domain LVLMs demonstrate particular weaknesses with hallucination sub-types like Symptom and Measurement across all tested modalities. Within the Med-LVLM category, CheXagent shows better overall accuracy on the CXR-VisHal benchmark. However, despite specialized medical training, these models frequently display higher hallucination rates on the MM-VisHal benchmark, where their accuracy is markedly lower than on the single-modality CXR-VisHal dataset. Notably, models such as LLM-CXR and CheXagent achieve exceptional performance on datasets aligned with their primary training domains, such as chest X-rays.

\begin{table}[t]
\centering
\resizebox{\columnwidth}{!}
{
\begin{tabular}{l|c|c|c|c|c}
\toprule
\rowcolor[HTML]{fbe7b6} \textbf{LVLM} & \textbf{CheXbert $\uparrow$} & \textbf{RadGraph $\uparrow$} & \textbf{RaTEScore $\uparrow$} & \textbf{Recall $\uparrow$} & \textbf{CHAIR $\downarrow$} \\
\midrule
GPT-4o & 21.71 & 10.28 & \underline{45.39} & 33.73 & 11.99 \\
LLaVA-NeXT 7B & 16.31 &  4.41 & 39.93 & 10.88 & 16.08 \\
LLaVA-NeXT 13B &  14.76 & 5.34 & 38.59 &  6.38 & 14.82 \\
MiniGPT-4 & 17.71 & 7.18 & 39.90 & 10.54 & 18.02 \\
\midrule
LLaVA-Med & 19.72 & 7.31 & 39.86 & 25.17 & 20.85 \\
LLaVA-Med-1.5 & 18.44 & 4.96 & 39.47 & 13.27 & 19.74 \\
LLM-CXR & 24.34 & 7.57 & 38.53 & 29.85 & 9.18  \\
Med-Flamingo & 17.50 & 5.83 &  35.87 & 17.52 &  23.96 \\
RadFM & 23.74 & 6.69 & 37.04 & 24.66 & 6.89 \\
CheXagent & \underline{30.32} & 12.35 & 43.18 & \underline{33.93} & \underline{6.88} \\
XrayGPT & 25.63 & \underline{12.88} & 44.45 & 30.87 & 12.84 \\
\midrule
\ours & \textbf{35.80} & \textbf{13.78} & \textbf{49.77} & \textbf{37.24} & \textbf{6.21} \\
\bottomrule
\end{tabular}}
\caption{{Open-ended evaluation on visual misinterpretation hallucination} (\underline{underlined}: second-best, \textbf{Bold}: best)}.
\label{tab:open_ended_VMH}
\vspace{-0.2in}
\end{table}

\begin{figure}[htp]
\centering
        \includegraphics[width=0.9\columnwidth]{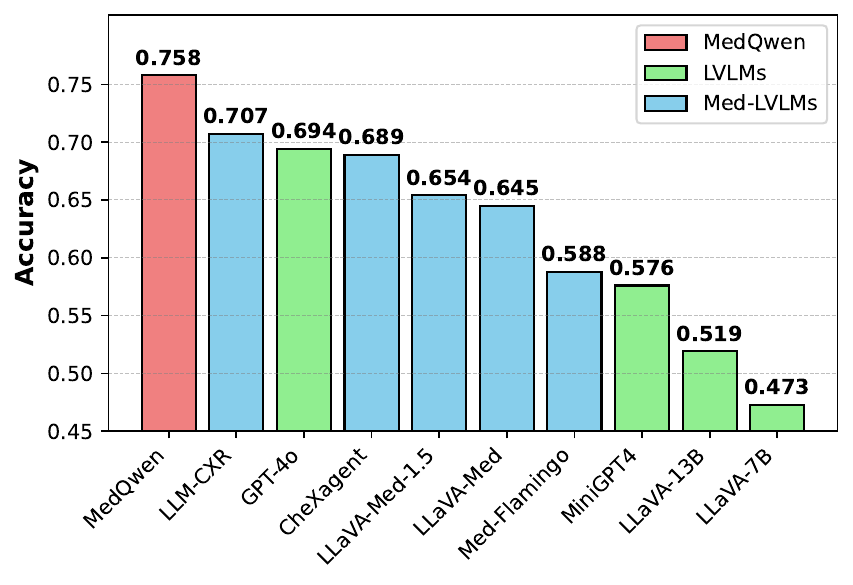}%
    \caption{Close-ended evaluation of knowledge deficiency hallucination in (Med)-LVLMs and the effectiveness of hallucination
mitigation methods.}
\label{fig:close_ended_KDH}
\end{figure}

\begin{table*}[t]
\centering

\vspace{-0.1in}
\small
\resizebox{\textwidth}{!}{
\begin{tabular}{l|cccccc|c}
\toprule
\rowcolor[HTML]{fbe7b6} & \multicolumn{6}{c|}{\textbf{Generation Metrics}} & \multicolumn{1}{c}{\textbf{Hallucination Score}} \\
\cline{2-8}
\rowcolor[HTML]{fbe7b6} \multirow{-2}{*}{\textbf{LVLM}} & \textbf{BertScore $\uparrow$} & \textbf{BLEU $\uparrow$} & \textbf{METEOR $\uparrow$} & \textbf{ROUGE-1 $\uparrow$}
  & \textbf{ROUGE-2 $\uparrow$} & \textbf{ROUGE-L $\uparrow$} & \textbf{$\mathcal{S}_h$ $\downarrow$} \\
\midrule
 GPT-4o & \underline{91.71}  & \underline{14.60} & 33.24 & \underline{47.77} & \underline{28.76} & \underline{40.76} & \underline{0.77 ± 0.81} \\
 LLaVA-NeXT 7B & 90.16 & 13.81 & 38.34 & 44.75 & 22.04 &  33.83 & 2.02 ± 1.20 \\
 LLaVA-NeXT 13B & 89.56 & 12.03 & 39.46 & 41.59 & 20.16 &  30.70 &  2.09 ± 1.09 \\
 MiniGPT-4 & 86.93 & 7.49 & 32.51 & 34.62 & 14.42 & 24.55 & 3.29 ± 1.25 \\
\midrule
 LLaVA-Med & 89.51 & 11.59 & 40.68 & 41.44 &  19.99 & 30.44 & 1.92 ± 1.02 \\
 LLaVA-Med-1.5 & 89.86 & 12.98 & \underline{41.30} & 43.52 & 21.37 & 32.27 & 1.76 ± 0.96 \\
 LLM-CXR & 87.98 & 4.15 &  16.55 & 28.71 & 13.05 & 23.28 & 2.52 ± 1.62 \\
 Med-Flamingo& 84.52 & 5.33 & 22.74 & 26.24 & 10.17  & 21.00  & 3.69 ± 1.10 \\
 RadFM&  79.66 & 7.99 & 25.51 & 32.63 & 14.22 & 24.14 & 2.30 ± 1.56 \\
 CheXagent& 87.82 & 4.65 & 16.85 & 28.07 & 15.38 & 23.75 & 2.08 ± 1.50 \\
 XrayGPT& 83.82 &  1.91 & 17.49 &  21.62 &  3.31 &  13.99 &  4.78 ± 0.63 \\
\midrule
 \ours & \textbf{92.75} & \textbf{17.49} & \textbf{44.60} & \textbf{49.37} & \textbf{31.91} & \textbf{43.12} & \textbf{0.63 ± 0.21} \\
\bottomrule
\end{tabular} }
\caption{{Results on open-ended evaluation of knowledge deficiency hallucination}. (\underline{underlined}: second-best, \textbf{Bold}: best) }
\label{tab:open_ended_KDH}
\vspace{-0.15in}
\end{table*}

\noindent
\textbf{Hallucination Evaluation on Open-Ended Tasks.} 
Table~\ref{tab:open_ended_VMH} presents the hallucination rates (CHAIR) from our open-ended evaluation, indicating that the majority of Med-LVLMs, including CheXagent, struggle to resist hallucinations when generating critical medical findings. Notably, our proposed method, \ours, achieves a lower hallucination rate and higher recall, outperforming both Med-LVLMs and general-domain LVLMs. Furthermore, the LLaVA-Med series demonstrates suboptimal performance in open-ended evaluations, exhibiting higher hallucination rates despite achieving higher recall compared to general-domain LVLMs. Report-specific metrics consistently reflect the overall quality of generated outputs: models with lower CHAIR scores and higher recall, such as CheXagent, GPT-4o, and XrayGPT, generally achieve superior performance.

\subsection{Knowledge Deficiency Hallucination}
Hallucination can also occur when the model correctly interprets the image, such as recognizing key organs and visual features, but lacks the comprehensive medical knowledge required for accurate diagnosis or clinical decision-making.

\noindent
\textbf{Hallucination Evaluation on Closed-Ended Tasks. } 
The results presented in Fig.~\ref{fig:close_ended_KDH} indicate that our proposed method, \ours, achieves a significant accuracy of 75.8\%, representing a 5.1\% improvement over LLM-CXR and a 6.4\% improvement over GPT-4o. Med-LVLMs, benefiting from their specialized medical knowledge, typically exhibit superior accuracy compared to general-domain LVLMs, with certain Med-LVLMs achieving performance comparable to GPT-4o. Despite this, their overall performance against knowledge-based hallucinations remains inadequate. These observations highlight that, even with training on varied multimodal medical datasets, conventional Med-LVLMs are prone to generating hallucinations when responding to diagnostic inquiries requiring specific domain knowledge. Therefore, our findings suggest that our MoE approach offers a more robust solution for mitigating knowledge hallucinations than traditional medical tuning methods.

\noindent
\textbf{Hallucination Evaluation on Open-Ended Tasks.} 
As evidenced in Table~\ref{tab:open_ended_KDH}, our proposed method, \ours, achieves the highest performance among the twelve evaluated LVLMs, exhibiting a remarkably low hallucination score of $S_h = 0.63$. Most Med-LVLMs demonstrate an increased propensity for hallucinations when interpreting complex medical knowledge, which aligns with observations from generation-focused metrics. Conversely, the LLaVA-Med series exhibits greater robustness against knowledge-based hallucinations. Regarding generation quality, most Med-LVLMs show relatively lower word-level coverage of ground truth compared to general-domain LVLMs such as LLaVA-NeXT 7B and 13B, indicating suboptimal content consistency. While GPT-4o consistently ranks second across the majority of metrics and surpasses certain specialized Med-LVLMs, XrayGPT performs poorly across most evaluation metrics, frequently generating irrelevant text and extraneous details in a medical-report style, which reflects its confined training focus primarily on medical summary generation.

\subsection{Context Misalignment Hallucination}
In addition to the evaluations outlined in previous sections, clinical practice necessitates that medical image interpretation aligns with the patient’s comprehensive medical history. This includes
critical factors such as treatment plans, diagnostic records, family history, and other relevant clinical data. However, existing benchmarks for assessing hallucination in medical imaging predominantly
focus on isolated image analysis, neglecting the broader clinical context integral to real-world practice. To address this gap and better align with the practical demands of the medical field, we
evaluate the model’s susceptibility to hallucinations by contextualizing medical images within the patient’s holistic medical background.

\begin{figure}[htp]
\centering
        \includegraphics[width=.9\columnwidth]{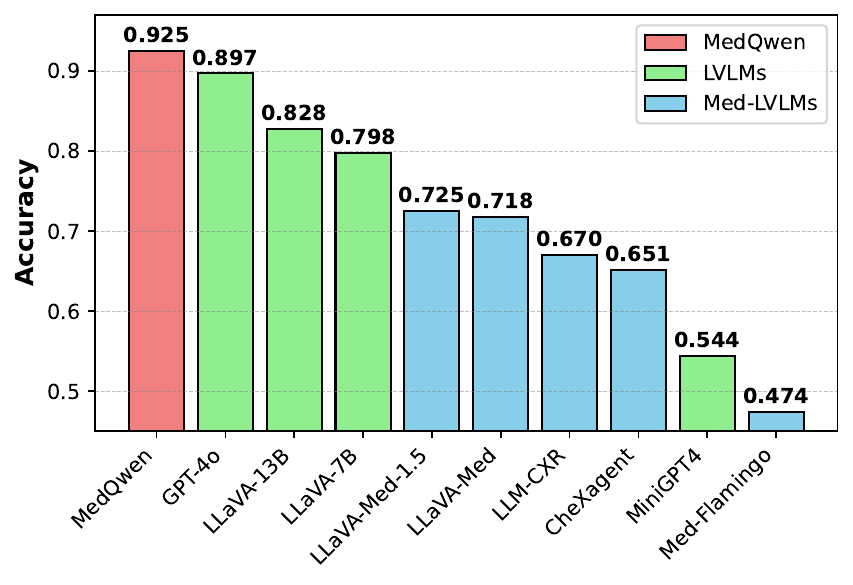}%
    \caption{Close-ended evaluation of context misalignment hallucination in (Med)-LVLMs and the effectiveness of hallucination
mitigation methods.}
\label{fig:close_ended_CMH}
\end{figure}

\noindent
\textbf{Mitigation Evaluation Results.} \\
Despite the superior performance of our proposed method (\ours) illustrated in Fig.~\ref{fig:close_ended_CMH}, general-domain LVLMs such as GPT-4o and LLaVA-NeXT 13B achieve higher accuracy than Med-LVLMs when responding to close-ended contextual questions. Notably, Med-Flamingo exhibits a below-average performance, suggesting that conventional fine-tuning on multimodal medical data might inadvertently impair the inherent reasoning capabilities of the foundational LVLMs. Consequently, Med-LVLMs become susceptible to hallucinations when confronted with intricate clinical contexts. Concurrently, these results underscore the efficacy of our proposed mixture-of-LoRA-experts approach in addressing the most challenging hallucination benchmarks.

\section{Parameter and FLOPs Analysis}\label{app:parameters}
We present a comprehensive parameter analysis comparing the complexity of various baseline models and our proposed method across different backbone architectures. The notation used for the architectural parameters is defined as follows:
\begin{itemize}
    \item $\textbf{H}$: Hidden dimension.
    \item $\textbf{r}$: Rank of the low-rank adaptation components.
    \item $\textbf{e}$: Number of experts in the MoE layer.
    \item $\textbf{L}$: Total number of layers in the model.
    \item $\textbf{V}$: Vocabulary size.
    \item $\textbf{P}$: Patch size in the Vision Transformer model.
    \item $\textbf{C}$: Number of input channels in the ViT model.
\end{itemize}

\noindent The detailed parameter budget analysis for the \ours\ and BiomedCLIP architectures is provided below:
\paragraph{MedQwen-7B:} $H=4096, r=32, e=2, L=28, V=151646$. The activation parameters are \texttt{q, k, v, up, down}.

\begin{enumerate}
    \item \textbf{FFT}:
    \begin{itemize}
        \item \textbf{Total Parameters}: \( (10.25H^2 + 2H)L + H + 2HV \)
        \begin{itemize}
            \item Embedding layer and LM head: \( 2HV \)
            \item Attention mechanism: \( 2.25H^2 \)
            \item MLP layer: \( 8H^2 \)
            \item RMSNorm (2 layers): \( 2H \)
            \item Additional RMSNorm (last layer): \( H \)
            \item Total per layer: \( 10.25H^2 + 2H \)
        \end{itemize}
    \end{itemize}
    \item \textbf{LoRA/MiLoRA/PiSSA/KASA}:
    \begin{itemize}
        \item \textbf{Total Parameters}: \( 11.58HLr \)
        \item \textbf{Proportion}: \( 0.78\% \)
    \end{itemize}
    \item \textbf{DoRA}:
    \begin{itemize}
        \item \textbf{Total Parameters}: \( (11.58Hr + 5)L \)
        \item \textbf{Proportion}: \( 0.78\% \)
    \end{itemize}
    \item \textbf{HydraLoRA}:
    \begin{itemize}
        \item \textbf{Total Parameters}: \( (4.91Hr + 6.66Hr/e + 6.66He)L \)
        \item \textbf{Proportion}: \( 0.58\% \)
    \end{itemize}
    \item \textbf{AdaMoLE}:
    \begin{itemize}
        \item \textbf{Total Parameters}: \( (11.58Hr + 6.66He + 6.66H)L \)
        \item \textbf{Proportion}: \( 0.82\% \)
    \end{itemize}
    \item \textbf{MoELoRA/Ours}:
    \begin{itemize}
        \item \textbf{Total Parameters}: \( (11.58Hr + 6.66He)L \)
        \begin{itemize}
            \item Attention mechanism: \( 4.25Hr+3He \)
            \item MLP layer: \( 7.33Hr+3.66He \)
            \item Total per layer: \( 6.66He + 11.58Hr  \)
        \end{itemize}
        \item \textbf{Proportion}: \( 0.81\% \)
    \end{itemize}
\end{enumerate}

\paragraph{BiomedCLIP:} $H=768, e=2, r=8, P=32, L=12, C=3$. The activation parameters include \texttt{q, k, v, o, fc1, fc2}.

\begin{enumerate}
    \item \textbf{FFT}:
    \begin{itemize}
        \item \textbf{Total Parameters}: \( (C+1)HP^2 + (12H^2 + 2H)L + H^2 + 3H + PH \)
        \item \textbf{Breakdown}:
        \begin{itemize}
            \item Embedding layer: \( PH+H+(C+1)P^2H \)
            \item encoder (L layers): \( (12H^2+2H)L \)
            \item LayerNorm (1 layers): \( 2H \)
            \item Pooler: \( H^2 \)
        \end{itemize}
    \end{itemize}
    \item \textbf{Full FT MoE}:
    \begin{itemize}
        \item \textbf{Total Parameters}: \( (C+1)P^2H + (12eH^2 + 2H + 9He)L + 3H + PH + H^2 \)
        \item \textbf{Proportion}: \( 760\% \)
    \end{itemize}
    \item \textbf{LoRA/PiSSA/MiLoRA}:
    \begin{itemize}
        \item \textbf{Total Parameters}: \( 18HLr \)
        \item \textbf{Proportion}: \( 1.49\% \)
    \end{itemize}
    \item \textbf{LoRA (rank=16)}:
    \begin{itemize}
        \item \textbf{Total Parameters}: \( 18HLr \)
        \item \textbf{Proportion}: \( 2.99\% \)
    \end{itemize}
    \item \textbf{LoRA (rank=32)}:
    \begin{itemize}
        \item \textbf{Total Parameters}: \( 18HLr \)
        \item \textbf{Proportion}: \( 5.98\% \)
    \end{itemize}
    \item \textbf{HydraLoRA}:
    \begin{itemize}
        \item \textbf{Total Parameters}: \( (9Hr + 9He + 9Hr/e)L \)
        \item \textbf{Proportion}: \( 1.58\% \)
    \end{itemize}
    \item \textbf{AdaMoLE}:
    \begin{itemize}
        \item \textbf{Total Parameters}: \( (18Hr + 9He + 9H)L \)
        \item \textbf{Proportion}: \( 2.33\% \)
    \end{itemize}
    \item \textbf{MoLoRA/Ours}:
    \begin{itemize}
        \item \textbf{Total Parameters}: \( (18Hr + 9He)L \)
        \item \textbf{Breakdown}:
        \begin{itemize}
            \item Attention mechanism: \( 8Hr+4He \)
            \item MLP layer: \( 10Hr+5He \)
            \item Total per layer: \( 18Hr + 9He \)
        \end{itemize}
        \item \textbf{Proportion}: \( 2.24\% \)
    \end{itemize}
\end{enumerate}

\subsection{FLOPs Analysis}\label{app:flops}
We conduct a detailed parameter analysis for each baseline and our proposed method, considering the underlying architectural backbones. The analysis utilizes the following set of variables: $H$, representing the model hidden dimension; $e$, indicating the number of experts; $r$, denoting the LoRA rank; $L$, representing the number of layers; $P$, denoting the patch size in ViT; $V$, signifying the vocabulary size; and $C$, representing the number of input channels in ViT. The subsequent sections provide the specific analysis for the MedQwen and BiomedCLIP architectures.

\paragraph{FLOPs for FT MoE\\ \\}

\hspace{-14pt} 1. MoE linear for $q$ and $o$:  
The FLOPs are calculated as
\[
2 \cdot \bigl( 2BsHe + k \cdot 2BsH^2 \bigr).
\]

\noindent 2. MoE linear for $k$ and $v$:  
Since MedQwen 7B's GQA reduces the number of heads for $k$ and $v$ to $1/8$ of $q$'s heads, the FLOPs are
\[
2 \cdot \bigl( 2BsHe + k \cdot 2BsH^2/8 \bigr).
\]

\noindent 3. FLOPs for $q \cdot k$ and $\text{score} \cdot v$:  
These remain independent of $k$, as we only upcycle the linear projection to $e$ copies. The FLOPs are
\[
2Bs^2H + 2Bs^2H.
\]

\noindent 4. MoE linear for $\text{down}$ and $\text{gate}$:  
Since MedQwen 7B uses SwiGLU FFN, the FLOPs are
\[
2 \cdot \bigl( 2BsHe + k \cdot 2BsH \cdot \tfrac{8}{3}H \bigr).
\]

\noindent 5. MoE linear for $\text{up}$:  
The FLOPs are
\[
2Bs \cdot \tfrac{8}{3}He + k \cdot 2Bs \cdot \tfrac{8}{3}H^2.
\]

\hspace{-14pt} Across $L$ layers, including the vocabulary embedding transformation, the total FLOPs of Full FT MoE are:
\begin{align*}
\mathrm{FLOPs}
&= BL \Bigl(
      \tfrac{52}{3} esH
      + \tfrac{41}{2} ksH^2
      + 4s^2H
    \Bigr)
\\
&\quad
+ 2BsHV.
\end{align*}

\paragraph{FLOPs for Ours/MoLoRA/HydraLoRA\\ \\}

\hspace{-14pt} 1. MoE linear for $q$ and $o$:  
The FLOPs are calculated as
\[
2B \cdot \bigl( 2sH^2 + 2esH + 2k(sHd + sHd) \bigr).
\]

\noindent 2. MoE linear for $k$ and $v$:  
Considering the effect of MedQwen 7B's GQA on $k$ and $v$, the FLOPs are
\[
2B \cdot \bigl( 2sH^2/8 + 2esH + 2k(sHd + sHd/8) \bigr).
\]

\noindent 3. FLOPs for $q \cdot k$ and $\text{score} \cdot v$:  
The FLOPs for these operations are
\[
2Bs^2H + 2Bs^2H.
\]

\noindent 4. MoE linear for $\text{down}$ and $\text{gate}$:  
Since MedQwen 7B uses SwiGLU FFN, the FLOPs are
\[
2B \cdot \bigl( 2sH \cdot \tfrac{8}{3}H + 2esH + 2k(sHd + s d \tfrac{8}{3}H) \bigr).
\]

\noindent 5. MoE linear for $\text{up}$:  
The FLOPs are
\[
2BsH \cdot \tfrac{8}{3}H + 2Bs \tfrac{8}{3}He
+ 2k \bigl( Bs \tfrac{8}{3}Hd + BsrH \bigr).
\]

\noindent Across $L$ layers, including the vocabulary embedding transformation, the total FLOPs of LoRA-MoE are:
\begin{align*}
\mathrm{FLOPs}
&= BL \Bigl(
      \tfrac{52}{3} esH
      + \tfrac{41}{2} sH^2
      + 4s^2H
      + \tfrac{69}{2} ksHd
    \Bigr)
\\
&\quad
+ 2BsHV
\end{align*}

\begin{table}[htbp]
\centering
\begin{adjustbox}{width=\linewidth}
\begin{tabular}{l}
\hline
\textbf{Case 1: Hallucination and mitigation in MM-VisHal} \\ \hline
\begin{tabular}[c]{@{}l@{}}
\includegraphics[width=3cm]{images/MM.jpg}
\end{tabular}
\begin{tabular}[c]{@{}l@{}}
\textbf{Question:} ``Does the brain appear abnormal in the image?''\\
\textbf{Ground truth:} ``Yes.''\\
\textbf{LLaVA-Med:} {\color{red}``No, there is no evidence of edema}\\{\color{red} in the brain in the MRI.''}\\
\textbf{Ours:} \textcolor{ForestGreen}{``Yes, there is evidence of edema in brain tissue.''}
\end{tabular}\\
\hline
\textbf{Case 2: Hallucination and mitigation in CXR-VisHal} \\ \hline
\begin{tabular}[c]{@{}l@{}}
\includegraphics[width=3cm]{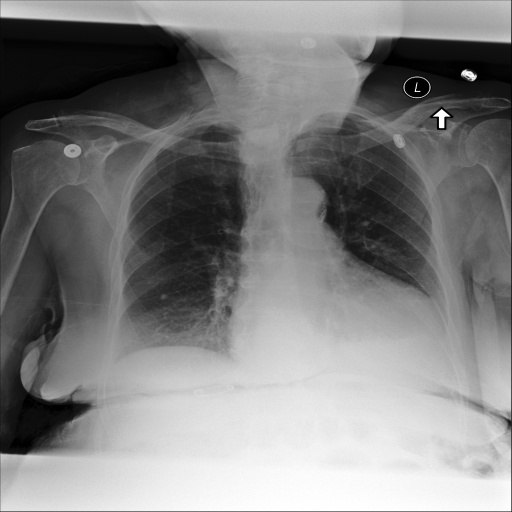}
\end{tabular}
\begin{tabular}[c]{@{}l@{}}
\textbf{Question:} ``Does the image show the presence of a pneumothorax?''\\
\textbf{Ground truth:} ``No.''\\
\textbf{LLaVA-Med:} {\color{red}``Yes, this image shows signs consistent with a}\\{\color{red} pneumothorax.''}\\
\textbf{Ours:} \textcolor{ForestGreen}{``No, there is no clear evidence of a pneumothorax.''}
\end{tabular}\\
\hline
\end{tabular}
\end{adjustbox}
\caption{Examples of hallucination and mitigation cases.}
\label{tab:example}
\end{table}

\begin{table}[t]
\centering
\begin{adjustbox}{width=\linewidth}
\begin{tabular}{l}
\hline
\textbf{Case studies on hallucination and mitigation in report generation.} \\ \hline
\begin{tabular}[c]{@{}l@{}}
\includegraphics[width=3cm]{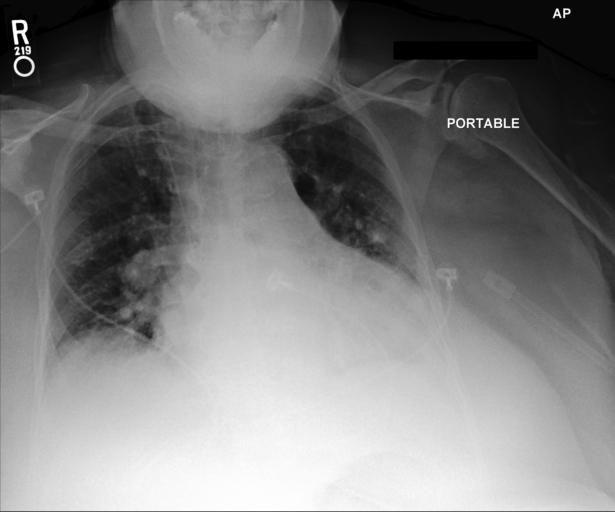}
\end{tabular}
\begin{tabular}[c]{@{}l@{}}
\textbf{Ground truth report:} ``The endotracheal tube is seen \\
terminating 5.6 cm above the carina, indicating proper\\
placement. The right internal jugular line ends at mid \\
superior vena cava (SVC). A feeding tube is visible \\
extending below the diaphragm into the stomach, although \\
the distal end is not within the radiographic view. There \\
are opacities in the left lower lung, suggesting a \\
combination of atelectasis and mild pleural effusion, \\
which is unchanged. The heart and mediastinal contours \\
appear mildly enlarged but stable.''\\[8pt]

\end{tabular}\\
\begin{tabular}[c]{@{}l@{}}
\textbf{LLaVA-Med}: ``Chest radiograph shows the \textcolor{ForestGreen}{cardiomediastinal silhouette}\\
is broadened, suggesting possible cardiomegaly or {\color{red}volume overload}. {\color{red}Hilar and} \\
{\color{red}mediastinal contours otherwise appear unremarkable.} {\color{red}No pulmonary edema},\\
free intraperitoneal air, acute rib or clavicular fracture is identified. There \\
is {\color{red}no right internal jugular catheter visible},
and {\color{red}no pneumothorax} is observed.''\\[8pt]
\textbf{GPT-4o}: ``The chest X-ray shows relevant measurements including the volume\\
of the lungs, the size of the heart, and the position of the diaphragm. \\
Additionally, the image may show any abnormalities or signs of symptoms \\
related to the chest, such as \textcolor{ForestGreen}{enlarged cardiac silhouette}, fluid accumulation, \\
or superimposed pneumonia.''\\[8pt]
\textbf{Ours}: ``The chest radiograph shows cardiomegaly, evidenced by an \textcolor{ForestGreen}{enlarged} \\
\textcolor{ForestGreen}{cardiac silhouette}, along with radiographic signs of pulmonary edema.\\
\textcolor{ForestGreen}{A central venous catheter is visible}, its tip projecting into the mid‑to‑lower\\
 portion of the superior vena cava.''\\[8pt]
\end{tabular}\\ \hline
\end{tabular}
\end{adjustbox}
\caption{Case studies highlighting hallucination and mitigation in medical report generation. The report generated by GPT-4o is generic and misses many details.}
\label{tab:report_example}
\end{table}

\section{Ablation details}\label{app:ablation}
Here, we provide a detailed explanation of the construction of each initialization method.
Suppose $h=min(m,n),t=\frac{h}{N}$
\begin{enumerate}
    \item \textbf{Ours (O)}:
    \begin{multline*}
        \mathcal{E}_r =
        \Bigl\{
            \bigl(U_{[:,\, k:k+d]},
            S_{[k:k+d,\,k:k+d]},
            V_{[k:k+d,\,:]}^\top\bigr)\\
            \;\big|\;
            k = (j-1)t, 
            j = 1, \dots, N
        \Bigr\}
    \end{multline*}

    \item \textbf{Principal (P)}:
    \begin{multline*}
        \mathcal{E}_r =
        \Bigl\{
            \bigl(U_{[:,\, k:k+d]},
            S_{[k:k+d,\,k:k+d]},
            V_{[k:k+d,\,:]}^\top\bigr)\\
            \;\big|\;
            k = (j-1)d, 
            j = 1, \dots, N
        \Bigr\}
    \end{multline*}

    \item \textbf{Minor (M)}:
    \begin{multline*}
        \mathcal{E}_r =
        \Bigl\{
            \bigl(U_{[:,\, k:k+d]},
            S_{[k:k+d,\,k:k+d]},
            V_{[k:k+d,\,:]}^\top\bigr)\\
            \;\big|\;
            k = h - jd, 
            j = 1, \dots, N
        \Bigr\}
    \end{multline*}

    \item \textbf{Random (R)}:
    \begin{multline*}
        \mathcal{E}_r =
        \Bigl\{
            \bigl(U_{[:,\, k:k+d]},
            S_{[k:k+d,\,k:k+d]},
            V_{[k:k+d,\,:]}^\top\bigr)\\
            \;\big|\;
            k = tj,
            t = \operatorname{random}\!\left(0, \frac{h}{d}-1\right),
            j = 0, \dots, N-1
        \Bigr\}
    \end{multline*}
\end{enumerate}


\subsection{Case Study for Hallucination and Mitigation}
\label{vfh_close_cases}
Table~\ref{tab:example} presents case studies from close-ended datasets on visual misinterpretation hallucinations. We provide the produced replies from our technique and baseline approach
in Cases 1 and 2 to visually show the efficacy of our strategy.
our method which enhances visual grounding through attention modification, successfully corrects the hallucinated responses. While it does not work for LLaVA-Med.
This study reinforces our findings that our mitigation method exhibit strength in hallucination mitigation, emphasizing the need for task-specific approach to improve Med-LVLM performance.

Also, Table~\ref{tab:report_example} shows case studies of open-ended report generation on visual misinterpretation hallucinations. In this example, our method effectively mitigates the hallucination
while even improving the recall of key findings. However, GPT-4o, while aiming to lower hallucination rates, significantly impact generation quality and recall, demonstrating the
challenges of balancing hallucination mitigation and report completeness.

\begin{table*}[t]
\centering
\resizebox{\textwidth}{!}
{
\begin{tabular}{l|cccc|c|cccc|c}
\toprule
\rowcolor[HTML]{fbe7b6}  & \multicolumn{5}{c|}{\textbf{MM-VisHal}} & \multicolumn{5}{c}{\textbf{CXR-VisHal}} \\
\cline{2-11}
\rowcolor[HTML]{fbe7b6} \multirow{-2}{*}{\textbf{LVLM}} & \textbf{Acc-A $\uparrow$} & \textbf{Acc-M$ \uparrow$} & \textbf{Acc-S $\uparrow$} & \textbf{Acc-R $\uparrow$} & \textbf{Acc $\uparrow$}
  & \textbf{Acc-A $\uparrow$} & \textbf{Acc-M $\uparrow$} & \textbf{Acc-S $\uparrow$} & \textbf{Acc-R $\uparrow$} & \textbf{Acc $\uparrow$} \\
\midrule
GPT-4o & \underline{0.775} & \underline{0.697} & \underline{0.708} & \underline{0.846} & \underline{0.741} & \underline{0.880} & \underline{0.595} & \underline{0.788} &\underline{0.921} & \underline{0.794}\\

LLaVA-NeXT 7B & 0.576 & 0.426 & 0.507 & 0.451 & 0.494 &
0.817 & 0.430  & 0.474 &0.362 & 0.518\\
LLaVA-NeXT 13B & 0.577 & 0.430 & 0.551 & 0.445 & 0.510 & 0.776 & 0.391  & 0.486 & 0.563 & 0.534 \\
MiniGPT-4 & 0.483 & 0.537 & 0.553 & 0.430 & 0.512& 0.341 &  0.301 & 0.573 &0.354 &  0.483 \\
\midrule
LLaVA-Med & 0.525 &  0.357 & 0.584 & 0.485 & 0.499 & 0.698 & 0.452 & 0.725 & 0.800 & 0.698 \\
LLaVA-Med-1.5 &  0.619 & 0.397 &  0.499 & 0.483 & 0.499 & 0.840 & 0.494 & 0.651 & 0.845 & 0.684\\
LLM-CXR & 0.486 & 0.460 & 0.513 &  0.314 &  0.461 &  0.681 & 0.504 & 0.743 & 0.403 & 0.675 \\
Med-Flamingo& 0.523 & 0.497 & 0.588 & 0.327 & 0.507 & 0.361  & 0.324  & 0.576 & 0.332 & 0.489 \\
CheXagent& 0.524 & 0.516 & 0.572 & 0.464 & 0.529 & 0.782 & 0.576 & 0.739 &0.851 & 0.739\\
\midrule
\ours& \textbf{0.883} & \textbf{0.798} & \textbf{0.896} & \textbf{0.847} & \textbf{0.854} & \textbf{0.893} & \textbf{0.702} & \textbf{0.839} & \textbf{0.916} & \textbf{0.825}\\
\bottomrule
\end{tabular} }
\caption{{Results on close-ended evaluation of visual misinterpretation hallucination}. We report Accuracy for each sub-type: Anatomy (\textbf{Acc-A}), Measurement (\textbf{Acc-M}), Symptom (\textbf{Acc-S}), Radiology Knowledge (\textbf{Acc-R}). We also report the overall accuracy (\textbf{Acc}). Higher accuracy in these evaluations indicates a stronger ability to resist hallucination. (\underline{underlined}: second-best, \textbf{Bold}: best)}
\label{tab:close_ended_VMH}
 \vspace{-0.1in}
\end{table*}

\section{Load Balancing Loss}
In standard MoE architectures \cite{fedus2022switch,dai2024deepseekmoeultimateexpertspecialization}, a balance loss, $\mathcal{L}_b$, is commonly employed to prevent routing collapse, thereby ensuring a uniform distribution of tokens across the available experts. This loss is formally defined as the dot product between the fractional load and the average routing probability for each expert:
\begin{align}
    \mathcal{L}_b &= \sum_{i=1}^E f_i P_i \label{eq:lb} \\
    f_i &= \frac{E}{kT} \sum_{t=1}^T \mathbf{1}\{\text{token } x_t \text{ is assigned to expert } i\}, \label{eq:f} \\
    P_i &= \frac{1}{T} \sum_{t=1}^T \text{softmax}(z^i(x_t))
\end{align}
Here, $T$ denotes the total number of tokens, and $\mathbf{1}(\cdot)$ is the indicator function. The term $f_i$ represents the normalized fraction of tokens explicitly routed to expert $i$, while $P_i$ quantifies the average router probability for that expert. By minimizing $\mathcal{L}_b$, the training objective explicitly promotes an equitable utilization of all experts. 

\section{Proof of Theoretical Results}
\newtheorem{Theorem2}{Theorem}

\subsection{Proof of Theorem~\ref{thm:single-align}}
\begin{tcolorbox}[colback=gray!20,colframe=gray]
\begin{Theorem}
Let $\eta_{\mathrm{FFT}}$ and $\eta_{\mathrm{LoRA}}$ denote the learning rates for
Full Fine-Tuning (FFT) and LoRA.
LoRA and Full FT behave equivalently when their initial weights satisfy
$\tilde{W}_0 \approx W_0$ and their scaled gradients match at every step,
i.e., $\eta_{\mathrm{LoRA}} \tilde{g}_t \approx \eta_{\mathrm{FFT}} g_t$.
See Eq.~\ref{def:eg} for formal definitions.
\end{Theorem}
\end{tcolorbox}
\begin{proof}
Define the LoRA effective weight as
$\tilde{W}_t = W_{\mathrm{init}} + s B_t A_t$
and its gradient as $\tilde{g}_t$.
Using SGD, the updates are:
\begin{align}
W_{t+1} &= W_t - \eta_{\mathrm{FFT}} g_t, \\
\tilde{W}_{t+1} &= \tilde{W}_t - \eta_{\mathrm{LoRA}} \tilde{g}_t.
\end{align}

\noindent\textbf{Base Case.}
At $t = 0$, we have $\tilde{W}_0 = W_0$.

\noindent\textbf{Inductive Step.}
Assume $\tilde{W}_t = W_t$ and the scaled gradients satisfy the alignment condition.
Then:
\begin{align}
\tilde{W}_{t+1} &= \tilde{W}_t - \eta_{\mathrm{LoRA}} \tilde{g}_t \\
&= W_t - \eta_{\mathrm{FFT}} g_t \\
&= W_{t+1}.
\end{align}
Thus, the weights remain identical for all $t$, establishing the alignment property.
\end{proof}

\subsection{Proof of Theorem~\ref{thm:MoE-ali}}
\begin{tcolorbox}[colback=gray!20,colframe=gray]
\begin{Theorem}
Let $\eta_{\mathrm{FFT}}$ and $\eta_{\mathrm{LoRA}}$ denote the learning rates employed in
Full FT MoE and LoRA MoE training.
For each expert $i \in \{1, \dots, N\}$, the two training procedures remain aligned when their
initial effective weights satisfy $\tilde{W}^{(0)}_{\,i} \approx W^{(0)}_{\,i}$ and their scaled gradients satisfy
$\eta_{\mathrm{LoRA}} \tilde{g}^t_i \approx \eta_{\mathrm{FFT}} g^t_i$ at each optimization step.
\end{Theorem}
\end{tcolorbox}
\begin{proof}
We demonstrate that these conditions ensure that LoRA MoE replicates the behavior of
Full FT MoE, particularly with respect to the routing mechanism.

\noindent\textbf{Base Case ($t = 0$).}
Because the Full FT MoE is constructed by upcycling, all expert weights satisfy $W^{(0)}_{\,i} = W^{(0)}$.
Thus the initialization condition implies $\tilde{W}^{(0)}_{\,i} \approx W^{(0)}$.
As both models use the same initialization seed, the router parameters at $t=0$ are identical, and both architectures produce the same routing assignments.

\noindent\textbf{Inductive Hypothesis.}
Assume that at iteration $t$ the equality $\tilde{W}^t_i = W^t_i$ holds for all experts and
that the routers coincide.

\noindent\textbf{Inductive Step.}
Using the scaled gradient alignment condition, we obtain
\begin{align}
\tilde{W}^{t+1}_i
&= \tilde{W}^t_i - \eta_{\mathrm{LoRA}} \tilde{g}^t_i \\
&\approx W^t_i - \eta_{\mathrm{FFT}} g^t_i \\
&= W^{t+1}_i.
\end{align}
Since the routers receive identical inputs and expert outputs, their updated parameters remain equal:
\begin{align}
\mathrm{MoE}(\mathbf{x})
&= \sum_{i=1}^{N} R(\mathbf{x})_i W_i(\mathbf{x})
= \sum_{i=1}^{N} R(\mathbf{x})_i \tilde{W}_i(\mathbf{x}),
\end{align}
which matches the LoRA MoE output.
Therefore, the routers remain aligned at step $t+1$.
Induction confirms that this holds for all $t$, and hence the two MoE models exhibit equivalent behavior.
\end{proof}

\subsection{Proof of Theorem~\ref{thm:router-mom}}
\begin{tcolorbox}[colback=gray!20,colframe=gray]
\begin{Theorem}
Consider the i.i.d.\ logits $z^i(\mathbf{x})$ and let $S_k(x)$
denote the indices of the $k$ largest among them, where $k \le N/2$.
Define the MoE weights as
\begin{equation}
R(\mathbf{x})_i =
\begin{cases}
\dfrac{\exp(z^i(\mathbf{x}))}{\sum_{j \in S_k(x)} \exp(z^j(\mathbf{x}))}, & \text{if } i \in S_k(x), \\[6pt]
0, & \text{if } i \notin S_k(x).
\end{cases}
\end{equation}
Then, for any pair $i \neq j$, we have:
\begin{align}
\mathbb{E}[R(\mathbf{x})_i] &= \frac{1}{N}, \\
\mathrm{Var}(R(\mathbf{x})_i) &= \frac{N-k}{kN^2}.
\end{align}
\end{Theorem}
\end{tcolorbox}
\begin{proof}
Because the logits are identically distributed and independent, permutations of
their indices do not alter the joint distribution.
Since the top-$k$ selection also respects such symmetry, the induced weights
$R(\mathbf{x})_i$ are exchangeable, which implies
\begin{align}
\mathbb{E}[R(\mathbf{x})_i] = \mathbb{E}[R(\mathbf{x})_j], \quad \forall i,j.
\end{align}
Using $\sum_{i=1}^N R(\mathbf{x})_i = 1$, we obtain:
\begin{align}
\sum_{i=1}^N \mathbb{E}[R_i] = 1
\quad\Longrightarrow\quad
\mathbb{E}[R_i] = \frac{1}{N}.
\end{align}
For the variance, observe:
\begin{align}
\mathrm{Var}(R_i) = \mathbb{E}[R_i^2] - \frac{1}{N^2}.
\label{eq:var_v2}
\end{align}
Expanding the identity $\left(\sum_i R_i\right)^2 = 1$ yields
\begin{align}
1 = N\,\mathbb{E}[R_i^2] + N(N-1)\,\mathbb{E}[R_i R_j].
\label{eq:E_v2}
\end{align}
To compute $\mathbb{E}[R_i R_j]$, define
\begin{align}
y_i =
\begin{cases}
\exp(z_i), & i \in S_k,\\
0, & i \notin S_k,
\end{cases}
\quad\text{so that}\quad
R_i = \frac{y_i}{\sum_{\ell \in S_k} y_\ell}.
\end{align}
Thus,
\begin{align}
R_i R_j = \frac{y_i y_j}{\left(\sum_{\ell \in S_k} y_\ell\right)^2}.
\end{align}
Since the probability that both $i$ and $j$ appear in the top-$k$ is $\binom{k}{2}/\binom{N}{2}$,
and upon selection, each is normalized by a sum of $k$ exponentials, symmetry gives
\begin{align}
\mathbb{E}[R_i R_j]
= \frac{k-1}{N(N-1)k}.
\end{align}
Substituting this into Eq~\eqref{eq:E_v2},
\begin{align}
N\,\mathbb{E}[R_i^2]
= 1 - \frac{k-1}{k},
\end{align}
hence
\begin{align}
\mathbb{E}[R_i^2] = \frac{1}{Nk}.
\end{align}
Inserting into Eq~\eqref{eq:var_v2} yields
\begin{align}
\mathrm{Var}(R_i)
= \frac{1}{Nk} - \frac{1}{N^2}
= \frac{N-k}{k N^2}.
\end{align}
\end{proof}

\subsection{Proof of Theorem~\ref{thm:residual}}
\begin{tcolorbox}[colback=gray!20,colframe=gray]
\begin{Theorem}
The solution to the optimization problem for the residual weight $W_{\text{res}}$:
\begin{equation}
    W_{\text{res}}^{+}
= \arg\min_{W_{\text{res}}}
\ \mathbb{E}_{\mathbf{x}}
\left\|
W_{\text{res}} - s \sum_{i=1}^N R(\mathbf{x})_i B_i^{(0)} A_i^{(0)}
\right\|^{2}.
\end{equation}
Its closed-form minimizer is given by:
\[
    W_{\mathrm{res}}^{+} = \frac{s}{N} \sum_{i=1}^{N} B_i^{(0)} A_i^{(0)}.
\]
\end{Theorem}
\end{tcolorbox}

\begin{proof}
The symbol \(W_{\mathrm{res}}^{+}\) denotes the optimizer of the stated problem.
By applying the linearity of expectation, the solution can be expressed as
\begin{align}
W_{\mathrm{res}}^{+}
    &= s\, \mathbb{E}_{\mathbf{x}}
        \left[
            \sum_{i=1}^{N} R(\mathbf{x})_i B_i^{(0)} A_i^{(0)}
        \right] \label{eq:v1a}\\
    &= s \sum_{i=1}^{N}
        \mathbb{E}_{\mathbf{x}}\!\left[R(\mathbf{x})_i\right]
        B_i^{(0)} A_i^{(0)} \label{eq:v1b} \\
    &= \frac{s}{N} \sum_{i=1}^{N} B_i^{(0)} A_i^{(0)},
\end{align}
where Eq~\eqref{eq:v1a} follows from linearity and Eq~\eqref{eq:v1b} uses Theorem~\ref{thm:router-mom}.
\end{proof}

\subsection{Proof of Theorem~\ref{thm:scale}}
\begin{tcolorbox}[colback=gray!20,colframe=gray]
\begin{Theorem}
Given the zero-initialization condition \(B_0 = 0\) and
\(
A_0 \sim U\!\left(-\sqrt{\tfrac{6}{n}}, \sqrt{\tfrac{6}{n}}\right),
\)
and the effective LoRA gradient:
\[
    \tilde{g}_t^i
    = s^2 \left( B_t^i {B_t^i}^{\top} g_t^i
    + g_t^i {A_t^i}^{\top} A_t^i \right),
\]
the optimal scaling factor \(s\) that minimizes the gradient mismatch
\(\lVert \tilde{g}_t^i - \eta g_t^i \rVert\) is:
\[
    s = \sqrt{\frac{3n\eta}{r}},
\]
where
\(\eta = \eta_{\text{FFT}} / \eta_{\text{LoRA}}\)
is the learning rate ratio required for update alignment.
\end{Theorem}
\end{tcolorbox}

\begin{proof}
We seek the optimal \(s\) by solving the minimization problem
\begin{equation} 
    s^{*} = \arg\min_{s} \lVert \tilde{g}_t^i - \eta g_t^i \rVert.
\end{equation}
We analyze the base case \(t=0\), where \(B_0 = 0\).
The objective simplifies to
\begin{equation} 
    \arg\min_{s} \left\lVert
        s^2 g_0^i A_0^{\top} A_0 - \eta g_0^i
    \right\rVert.
\end{equation}
To obtain a closed-form solution, we invoke the Law of Large Numbers 
and replace the random matrix \(A_0^{\top} A_0\) with its expected value.
The initialization scheme (Leaky ReLU variance \cite{xu2015empiricalevaluationrectifiedactivations}) provides entries in \(A_0\) with variance
\(\sigma_A^2 = 1/(3n)\). The expectation of the matrix product is
\begin{equation} 
    \mathbb{E}_{A_0}[A_0^{\top} A_0]
    = r \sigma_A^2 \mathbf{I}_{n \times n}
    = \frac{r}{3n} \mathbf{I}_{n \times n}.
\end{equation}
Substituting this expected value into the minimization objective and assuming
the optimal solution corresponds to setting the error term to zero, we require
\begin{equation} 
    s^2 g_0^i \left( \frac{r}{3n} \mathbf{I} \right)
    \approx \eta g_0^i.
\end{equation}
For this approximate equality to hold, the scalar coefficients must match:
\begin{equation} 
    s^2 \frac{r}{3n} = \eta
    \quad\Longrightarrow\quad
    s = \sqrt{\frac{3n\eta}{r}}.
\end{equation}
This result, derived from the first step and based on a strong expectation
approximation, provides the theoretically optimal scaling factor for gradient
alignment. Its applicability can be extended to subsequent steps due to the
typically small magnitude of relative weight changes in PEFT~\cite{hulora}.
\end{proof}

\begin{figure*}
  \centering
        \includegraphics[width=\textwidth]{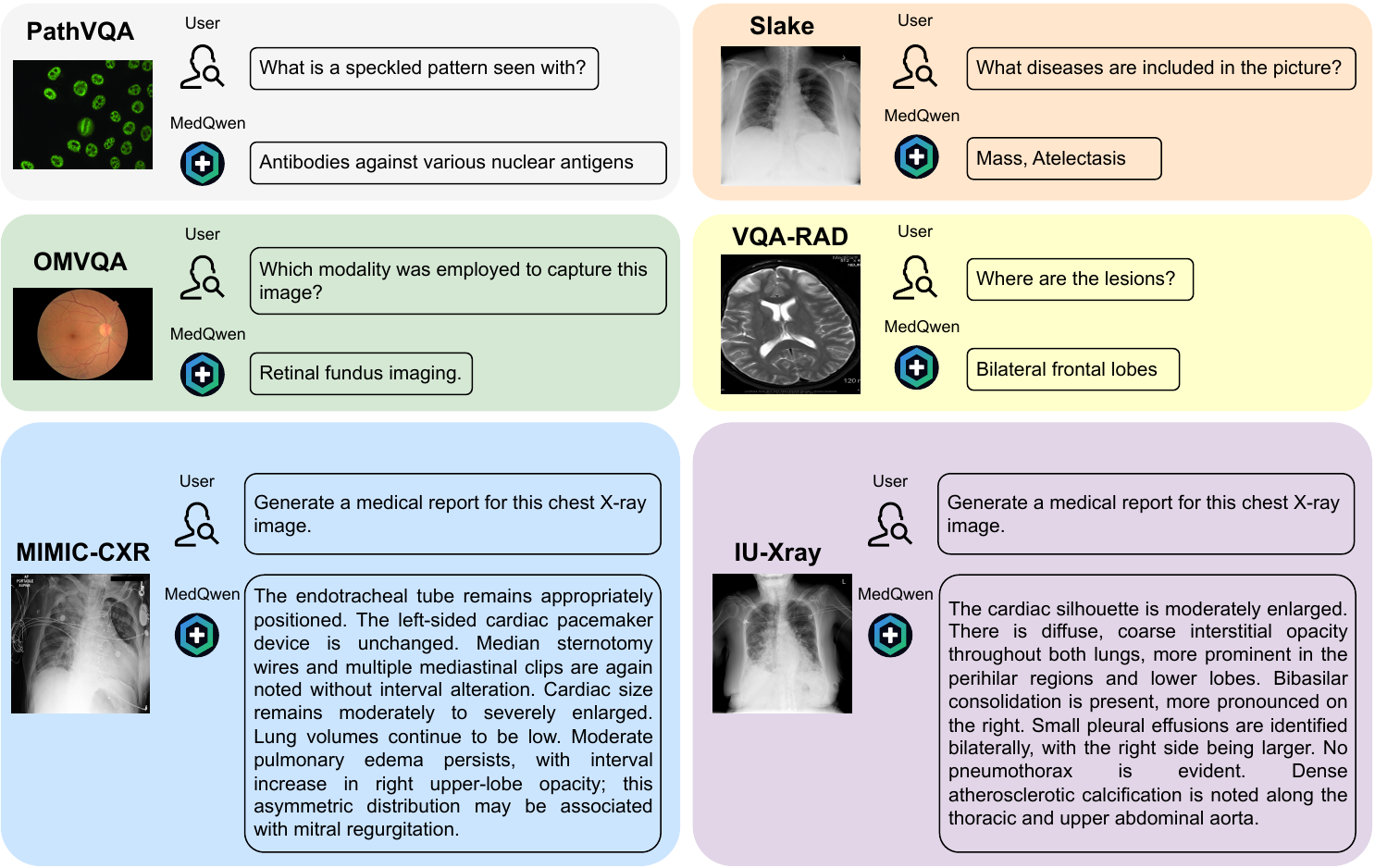}
        \captionof{figure}{A visualization of \ours, demonstrating its ability to process multiple modalities. The top two rows show VQA results, and the bottom row shows report generation.}
    \label{fig:vis}
    \vspace{-1em}
\end{figure*}

%% file: tables/data_summary.tex
\begin{table*}[!ht]
\centering %
\footnotesize
\resizebox{.8\textwidth}{!}{%
\begin{tabular}{|c|l|c|}
\hline
\textbf{Stage} & \textbf{Data Source} & \textbf{Sample Size} \\
\hline
Stage 1 & llava\_med\_alignment\_500k.json & 500K \\
\hline
Stage 2 & instruct\_60k\_inline\_mention & 60K \\
\hline
\multirow{2}{*}{Stage 3} & RAD-VQA, SLAKE, Path-VQA, OmniMedVQA, IU-Xray, & \multirow{2}{*}{2.4M} \\
&  MIMIC-CXR, Harvard-FairVLMed, Quilt-1M, PMC-OA & \\
\hline
\end{tabular}
}
\caption{Summary of Data Utilized Across Training Stages}
\label{tab:data_summary}
\end{table*}

%% file: tables/datasets.tex
\begin{table*}[h]
\tablestyle{-27pt}{1.1}
\addtolength{\tabcolsep}{+30pt}
\resizebox{\textwidth}{!}{%
\begin{tabular}{|c|c|c|c|}
\hline
\textbf{Category}                     & \textbf{Name} & \textbf{Modality}                                                                                                                                                                                  & \textbf{\# Image/QA text} \\ \hline
\multirow{10}{*}{Visual Question Answering }                 & SLAKE \cite{liu2021slake}     & CT, MRI, X-ray & 642/14028             \\ \cline{2-4}
& VQA-RAD~\cite{lau2018dataset}   & CT, MRI, X-ray & 315/3515 \\ \cline{2-4}
& PathVQA~\cite{he2020pathvqa}        & Histopathology                 & 4998/32799              \\ \cline{2-4}
& OmniMedVQA~\cite{hu2024omnimedvqa}      & \begin{tabular}[c]{@{}c@{}}
CT, MRI, X-ray, \\ 
Ultrasound, Fundus, \\
Histopathology, OCT, \\
Dermoscopy, Colposcopy, \\
Digital Photography, \\ 
Infrared Reflectance Imaging,\\
Endoscopy, Microscopy Images\end{tabular}                                                                        & 118010/127995              \\ \cline{2-4}
& Quilt-1M~\cite{ikezogwo2023quilt}         & Histopathology                                                                                         & 1M/1M            \\ \cline{2-4}
& PMC-OA~\cite{lin2023pmc}       & \begin{tabular}[c]{@{}c@{}}CT, MRI, X-ray, \\
Ultrasound, Endoscopy, \\
Microscopy Images
\end{tabular}                                                                                     & 1.6M/1.6M   \\ \cline{2-4}

& Harvard-FairVLMed~\cite{luo2024fairclip}    & Fundus                                                     & 10000/10000                \\ \hline
\multirow{2}{*}{Report Generation}                          & IU-Xray~\cite{demner2016preparing}         & X-ray                                                                                                       & 8121/3996               \\ \cline{2-4}
& MIMIC-CXR~\cite{johnson2019mimic}         & X-ray                                            & 227827/227835                \\ \hline
Image Classification    & UniMed~\cite{khattak2024unimed}         & \begin{tabular}[c]{@{}c@{}}
CT, MRI, X-ray,\\
Ultrasound, Histopathology
\end{tabular}                                                                                                          & 5.3M/5.3M                \\ \hline
\end{tabular}
}
\caption{An overview of the datasets used in this study.}%
\label{tab:dataset_info}
\end{table*}
        

%% file: tables/hyperparameter.tex
\begin{table}[ht]
\centering
\small
\resizebox{.45\textwidth}{!}{
\begin{tabular}{l|c}
\toprule
\textbf{Hyperparameter}                     & \textbf{Visual Question Answering} \\ \midrule
\textbf{Batch Size}                         & 16                    \\
\textbf{Rank}                               & 32                    \\
\textbf{Alpha}                              & 64                    \\
\textbf{Optimizer}                          & AdamW                 \\
\textbf{Warmup Steps}                       & 100                   \\
\textbf{Dropout}                            & 0.05                  \\
\textbf{Learning Rate}                      & 1e-4                  \\
\bottomrule
\end{tabular}}
\caption{Hyperparameters of the VQA task for \ours.}
\label{tab:vqa_hyper}
\end{table}

\begin{table}[ht]
\centering
\small
\resizebox{.45\textwidth}{!}{
\begin{tabular}{l|c}
\toprule
\textbf{Hyperparameter}                     & \textbf{Medical image classification} \\ \midrule
\textbf{Batch Size}                         & 256                    \\
\textbf{Rank}                               & 8                    \\
\textbf{Alpha}                              & 16                    \\
\textbf{Optimizer}                          & AdamW                 \\
\textbf{Warmup Steps}                       & 100                   \\
\textbf{Dropout}                            & 0.05                  \\
\textbf{Learning Rate}                      & 1e-4                  \\
\bottomrule
\end{tabular}}
\caption{Hyperparameters of the image classification task.}
\label{tab:cl_hyper}
\end{table}

%% file: tables/rout.tex
\begin{table}[h]
\centering
\caption{Comparison of routing strategies on average performance.}
\setlength{\tabcolsep}{5pt}
\footnotesize
\begin{tabular}{lc}
\toprule
\textbf{Routing Strategy} & \textbf{Avg} \\
\midrule
Ours (top-$k=2$)        & \textbf{68.24} \\
Top-$p$ ($p=0.25$)      & 66.40 \\
Top-$k$ + Shared Expert & 65.67 \\
\bottomrule
\end{tabular}
\label{tab:routing_comparison}
\end{table}

%% file: tables/qwen_pro.tex
\begin{table}[ht]
\centering
\caption{MedQwen-e vs properly-scaled MedQwen.}
\resizebox{\columnwidth}{!}{%
\begin{tabular}{lccccc}
\toprule
\textbf{Method} & \textbf{IU-Xray} & \textbf{MIMIC-CXR} & \textbf{Harvard-FVLM} & \textbf{Quilt-1M} & \textbf{PMC-OA} \\
\midrule
\textbf{MedQwen}   & 90.33 & 84.68 & 88.43 & 73.74 & 66.32 \\
\textbf{MedQwen-e} & 90.29 & 84.53 & 88.38 & 73.51 & 66.22 \\
\bottomrule
\end{tabular}%
}
\label{tab:qwen_pro}
\end{table}